\documentclass{article}
\usepackage{multirow}
\usepackage{amsmath,graphics,amssymb,bbm,mathtools,amsfonts,latexsym,color,pifont,soul,graphicx,subcaption,placeins}
\newcommand{\setunder}[2]{\underset{#2}{#1}}
\makeatletter
\def\argmax{\argmax@}
\def\argmax@{%
    \@ifnextchar{_}{\argmax@sub}{\text{\upshape\mdseries argmax} \,}
}%
\def\argmax@sub#1#2{\setunder{\text{\upshape\mdseries argmax}}{#2} \,}
\makeatother
\makeatletter
\def\argmin{\argmin@}
\def\argmin@{%
    \@ifnextchar{_}{\argmin@sub}{\text{\upshape\mdseries argmin} \,}
}%
\def\argmin@sub#1#2{\setunder{\text{\upshape\mdseries argmin}}{#2} \,}
\makeatother
\usepackage{amsthm}
\newtheorem{theorem}{Theorem}
\theoremstyle{definition}
\newtheorem{definition}{Definition}
\usepackage{dsfont}
\newcommand{\Exp}{\mathds{E}}

\usepackage{bm}

\newcommand{\Prob}{\mathds{P}}

\usepackage{bmpsize}
\begin{document}
\noindent \today
\nocite{*}
\begin{center}
  {\Large Advanced posterior analyses of hidden Markov models: \vspace{2mm} \\
  finite Markov chain imbedding and hybrid decoding} \vspace{3mm} \\
Zenia Elise Damgaard Bæk$^{*1}$, Moisès Coll Macià$^{*2}$, Laurits Skov$^{3}$, Asger Hobolth$^{\#4}$ 
\end{center}
\begin{description} 
  \item{$*$} Shared first authorship (authors contributed equally) 
  \item{$\#$} Corresponding author 
\end{description}
\begin{enumerate}
\item National Centre for Register-based Research, Department of Public Health, Aarhus University, Aarhus V, Denmark.
zenia.ncrr@au.dk. ORCID 0009-0008-6098-3076. 
\item Bioinformatics Research Centre, Aarhus University, Aarhus C, Denmark. moicoll@birc.au.dk. ORCID 0000-0002-7328-3553. 
\item Section for Evolutionary Genomics, Globe Institute, University of Copenhagen, Copenhagen, Denmark. laurits.skov@sund.ku.dk. ORCID 0000-0001-9582-0391. 
\item Department of Mathematics, Aarhus University, Aarhus C, Denmark. asger@math.au.dk. ORCID 0000-0003-4056-1286. 
\end{enumerate}
\section*{Abstract}
Two major tasks in applications of hidden Markov models are to (i) compute distributions of summary statistics of the hidden state sequence, and (ii) decode the hidden state sequence. We describe finite Markov chain imbedding (FMCI) and hybrid decoding to solve each of these two tasks. In the first part of our paper we use FMCI to compute posterior distributions of summary statistics such as the number of visits to a hidden state, the total time spent in a hidden state, the dwell time in a hidden state, and the longest run length. We use simulations from the hidden state sequence, conditional on the observed sequence, to establish the FMCI framework. In the second part of our paper we apply hybrid segmentation for improved decoding of a HMM. We demonstrate that hybrid decoding shows increased performance compared to Viterbi or Posterior decoding (often also referred to as global or local decoding), and we introduce a novel procedure for choosing the tuning parameter in the hybrid procedure. Furthermore, we provide an alternative derivation of the hybrid loss function based on weighted geometric means. We demonstrate and apply FMCI and hybrid decoding on various classical data sets, and supply accompanying code for reproducibility.
\vspace{2mm} \\
\noindent {\bf Key words:} Artemis analysis, decoding, finite Markov chain imbedding, hidden Markov model, hybrid decoding, pattern distributions. 
\section{Introduction}
Hidden Markov models (HMMs) constitute an enormously popular class of models, and are useful in many different contexts. The basic theory, various extensions and numerous applications of HMMs can be found in \cite{Zucchini2016}. The classical book by \cite{durbin1998biological} introduced and demonstrated the versatile framework in the field of genomics, and persists in being a major source of inspiration. Gene finding \cite{burge1997prediction, hobolth2005applications} is among the most prominent early applications. We refer to \cite{dutheil2017hidden} for a review on hidden Markov models in population genomics. A recent application of HMMs in the study of human evolution is the {\rm hmmix} software ({\verb+https://pypi.org/project/hmmix/+}), which identifies Neanderthal and Denisovan admixture fragments in modern humans \cite{SkovHuiShchurHobolthScallySchierupDurvin2018}. The framework is a two-state Poisson HMM that identifies archaic sequences in modern human genomes without requiring a reference archaic genome, assuming instead an outgroup devoid of archaic sequence. This application was the original motivation for the present work, and we apply all the methods explained here in an acoompanying paper \cite{CollMacia2025}. 

The standard statistical inference procedure for analysis of sequential data using HMMs often begins with parameter estimation using the Baum-Welch algorithm, which is a special case of the Expectation-Maximization algorithm (see e.g. Chapter~4 in \cite{Zucchini2016}). An alternative to the Baum-Welch algorithm is simple numerical maximization of the likelihood, and this procedure is advocated in \cite{turner2008direct}. Next, decoding is applied to uncover the hidden states (e.g. location of promotors, introns and exons in a genome) using either the Viterbi (global) or Posterior (local) decoding methods as explained in e.g. Chapter~5 in \cite{Zucchini2016}. The decoded regions are then analyzed to specify quantities such as the positions and proportion of exons in a genome, for example. 

Despite the utility of HMMs, the reliability of its results is constrained by the inherent limitations of both the Viterbi and Posterior decoding methods. The Viterbi algorithm (see e.g. Section~5.4.2 in \cite{Zucchini2016}), which identifies the globally most likely hidden state sequence , tends to be conservative and less prone to state transitions. In contrast, Posterior decoding (see e.g. Section~5.4.1 in \cite{Zucchini2016}) maximizes the probability of each individual state locally, often resulting in frequent and sometimes rather unlikely state transitions (in some special cases even suggesting state transitions with a probability of zero according to the model). This creates a severe dilemma in applications where the researcher must choose between the globally most likely hidden state sequence with low position-wise accuracy (Viterbi decoding) or a locally highly accurate solution with a globally low probability (Posterior decoding). 

Moreover, since both decoding methods provide a single solution for decoding (i.e., a single division of a genome into promotors, introns, exons and inter-genic regions), they provide no information about the uncertainty associated with these estimates. These simplistic summaries, combined with the methodological limitations of each approach, introduce biases in the estimation of summary statistics derived from the decoded sequence. Instead, a more fruitful procedure is to take all paths through the hidden states into account, and weight the paths according to their posterior probabilities given the model parameters and the observed sequence. 

To address these issues, we extend the standard statistical inference procedure for HMMs in two different directions. First, we improve the estimation of summary statistics by implementing the finite Markov chain imbedding (FMCI) for HMMs suggested by \cite{AstonMartin2007}. FMCI gives access to the distribution of patterns in the hidden state sequence. In our application the patterns are the number of jumps between any two different hidden states, the number of positions in each hidden state, the run length in a hidden state, and the maximum run length in a hidden state. Second, we improve the decoding of the hidden state sequence by applying the hybrid risk approach suggested originally by \cite{LemberKolydenko2014} and recently investigated by \cite{KuljusLember2023}. Hybrid decoding is the hidden state sequence that maximizes a weighted version of pointwise accuracy (local decoding) and posterior probability (global decoding). We suggest a novel analysis procedure, based on the so-called Artemis plot, to learn the right weight between the losses associated with global and local decoding. We demonstrate that by integrating these two advanced analysis methods (FMCI and hybrid decoding) into the standard HMM analyses we achieve more accurate insights into the posterior of the hidden state sequences, and obtain a more reliable decoding of hidden states. 

The remaining part of the paper is organized as follows. 
In Section~\ref{Sec:Motivation} we motivate the use of FMCI, Artemis analysis and hybrid decoding for improved statistical inference of HMMs. 
In Section~\ref{Sec:FMCI} we describe in detail the FMCI framework for finding the posterior distribution of summary statistics. In Section~\ref{Sec:HybridDecoding} we describe the hybrid method and carry out a large simulation study to investigate the Artemis Plot and situations where hybrid decoding is particularly useful. 
In this paper we mainly concentrate on the basic theory, classical illustrative data sets, and simple HMMs. In Section~\ref{Sec:Discussion} we discuss how FMCI and hybrid decoding can be extended to more complex pattern distributions, more general HMMs, and large-scale applications.   
Our results and figures can be reproduced from the code in the GitHub repository {\tt https://github.com/AsgerHobolth/PosteriorHMM/}.
\section{Motivation} \label{Sec:Motivation}
The purpose of this section is to motivate finite Markov chain imbedding (FMCI; Section~\ref{sec:FMCImotivate}) and hybrid decoding, and to introduce the Artemis Plot (Section~\ref{sec:ArtemisMotivate}).
\subsection{Finite Markov chain imbedding and fetal movement counts} \label{sec:FMCImotivate}
In Figure~\ref{FetalLambFig} we analyze the number of fetal lamb movements in five-second intervals. 
The data in Figure~\ref{FetalLambFig}a is from page~76-77 in~\cite{Zucchini2016}, and can also be found on page 123-124 in \cite{guttorp2018stochastic}. We analyze the data using a 2-state Poisson HMM. The original source is \cite{LerouxPuterman1992} who showed that a Poisson HMM is a plausible mechanism for the physiological process. 

The maximum likelihood estimates for the 2-state hidden Markov chain and Poisson emission probabilities are given by
\begin{equation*}
    \pi = (1, 0),
    \;\; 
    \Gamma =
    \begin{pmatrix}
        0.989 & 0.011 \\
        0.287 & 0.703
    \end{pmatrix},
    \;\;
    {\rm and} \;\;
    \lambda = (0.278,3.217).
\end{equation*}
\begin{figure}[!htb]
  \centering
  \includegraphics[scale=0.8]{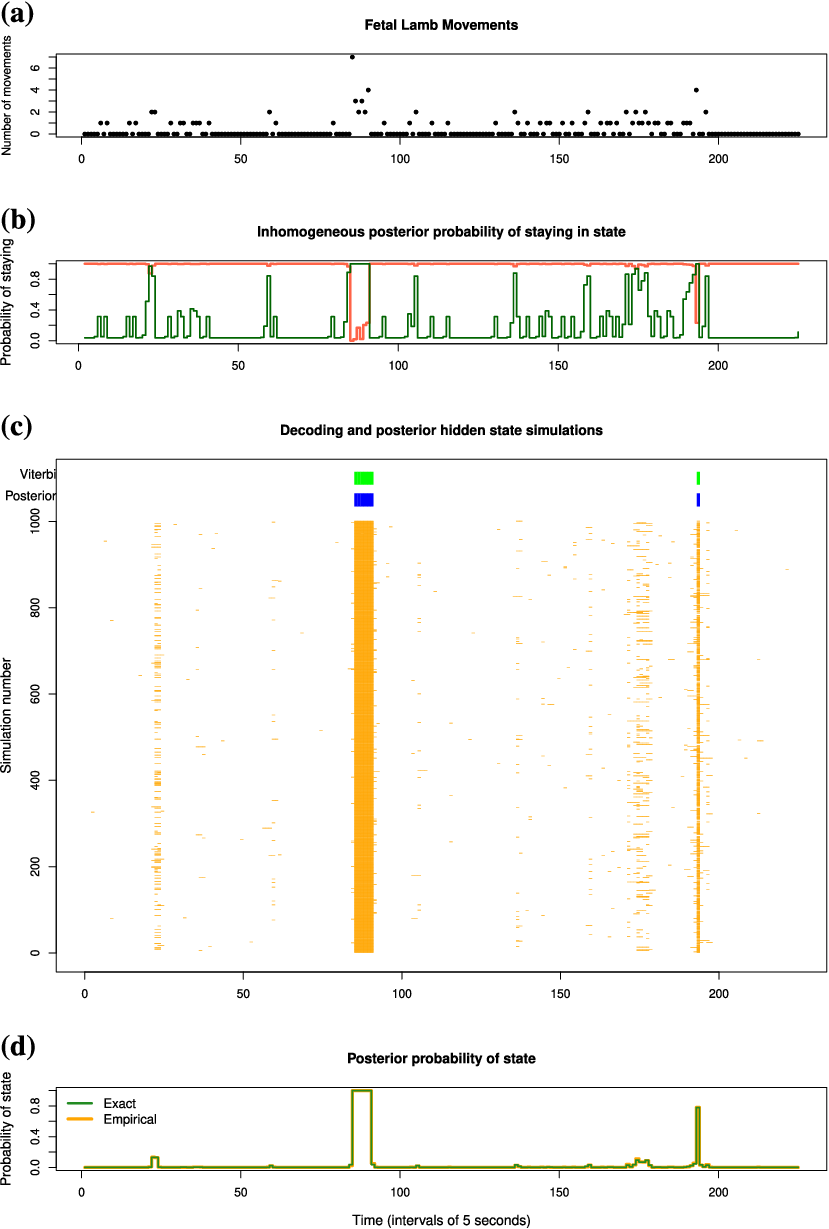}
  \caption{
  Illustration and analysis of fetal movement count data. {\bf (a)} The data consists of the number of fetal lamb movements in 5 second intervals. {\bf (b)} Illustration of Theorem~1 in Section~\ref{Sec:FMCI} of the manuscript. The red line is the probability for the hidden Markov chain of staying in state~1 at any time interval along the sequence given the observed sequence. The green line is the probability of staying in state~2 given the observed sequence. {\bf (c)} Top: Viterbi and Posterior decoding are identical for this data. Bottom: Samples of 1000 hidden state sequences from the conditional Markov chain illustrated in {\bf (b)}. The samples give a more nuanced and informative picture of the conditional hidden Markov chain than Viterbi or Posterior decoding. {\bf (d)} Pointwise empirical frequency of the 1000 samples is almost identical to the exact posterior probabilities found from the forward-backward tables.     
  } 
  \label{FetalLambFig}
\end{figure}

The crucial result for finite Markov chain imbedding to apply is the fact that the hidden Markov chain, conditional on the observed sequence, is an inhomogeneous Markov chain. This result is shown in Theorem~\ref{thm:inhomo} in Section~\ref{Sec:FMCI}. In Figure~\ref{FetalLambFig}b we show the inhomogeneous transition probabilities for staying in each of the two hidden states in each of the time intervals. The red line is the state with few movements (state~1) and the green line is the state with many movements (state~2). The two regions around interval number 85 and 195 stand out when considering the probability of staying in state~1. These two regions are also identified by Viterbi and Posterior decoding (with a threshold of 0.5) as belonging to state~2. Outside of these two regions, the probability for staying in state~1 is generally close to one, but there are more than 20 instances along the sequence where the probability of a transition from state~1 to state~2 is larger than 1\%. For example, around interval number 20 and interval number 175, it is rather likely to make a transition, and in this case the run in state~2 can easily span a few intervals as indicated in the green line in Figure~\ref{FetalLambFig}b. This property is also visible in the samples from the inhomogeneous hidden Markov chain in Figure~\ref{FetalLambFig}c. In Figure~\ref{FetalLambFig}d we show the empirical (based on the samples) and exact (based on the forward-backward tables; see e.g. Section~5.4.1 in \cite{Zucchini2016}) probability of being in state~1 in each time interval. The four regions around interval numbers 20, 85, 175 and 195 clearly stand out with a higher state~2 probability than the remaining intervals.

The samples can be used for more than summarizing the posterior probability of the hidden state sequence in each time interval. The yellow stair plots in  Figure~\ref{FetalLambFMCIFig} are empirical summary statistics for four different variables: the number of transitions from state~1 to state~2 (Figure~\ref{FetalLambFMCIFig}a), the number of intervals in state~2 (Figure~\ref{FetalLambFMCIFig}b), the expected number of runs of certain lengths (Figure~\ref{FetalLambFMCIFig}c), and the longest run (Figure~\ref{FetalLambFMCIFig}d). The green plots in Figure~\ref{FetalLambFMCIFig} are analytical calculations of the summary statistics. The distributions are calculated using finite Markov chain imbedding (FMCI), which is explained in detail in Section~\ref{Sec:FMCI}.    
\begin{figure}[!htb]
  \centering
  \includegraphics[scale=0.55]{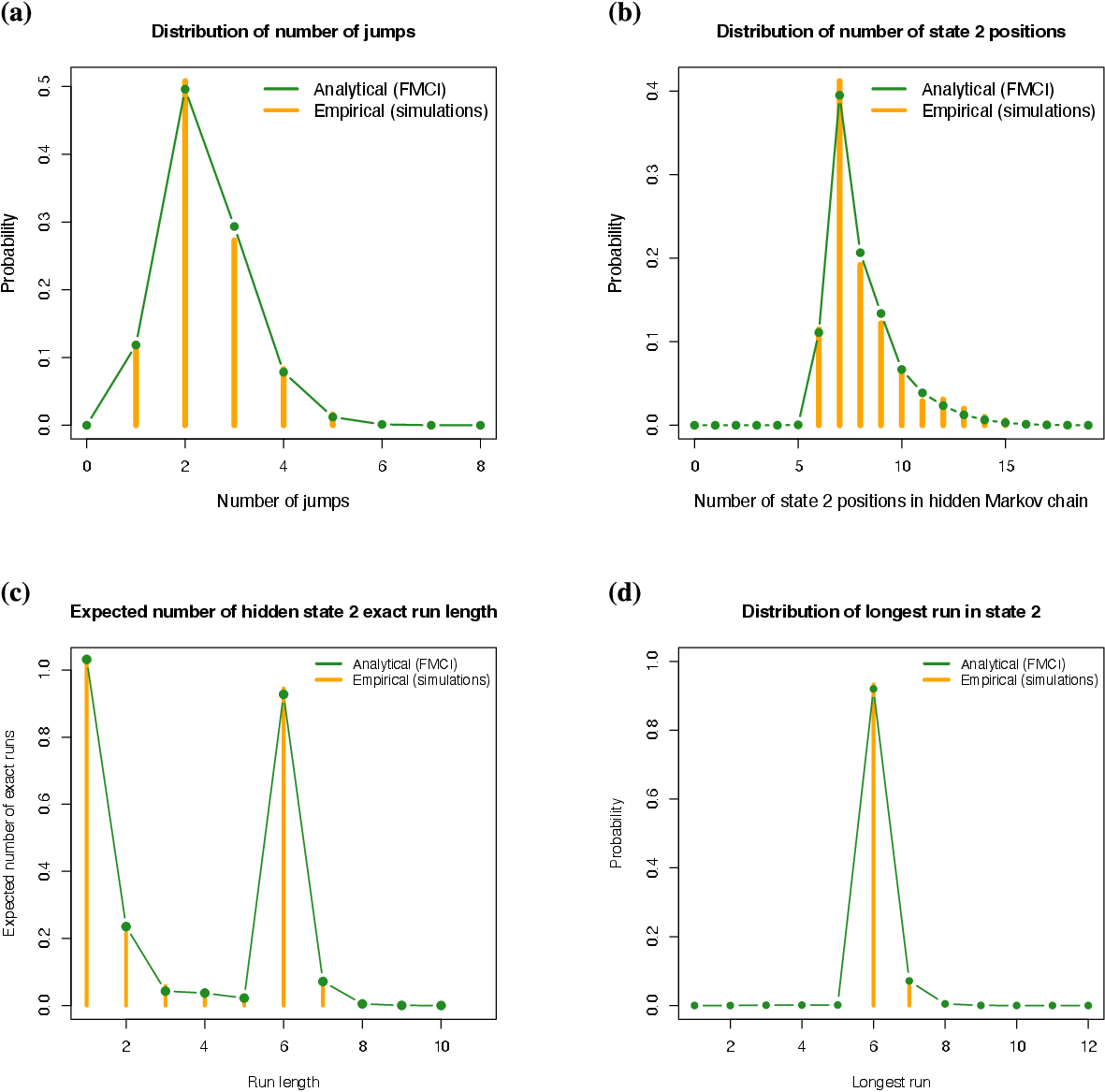}
  \caption{
  FMCI for fetal movement count data. Analytical plots are from FMCI and empirical plots are based on summary statistics from 1000 samples. {\bf (a)} Posterior distribution of the number of transitions from state~1 to state~2. {\bf (b)} Posterior distribution of the number of intervals in state~2. {\bf (c)} Expected number of run lengths in state~2. {\bf (d)} Posterior distribution of the longest run in state~2.      
  } 
  \label{FetalLambFMCIFig}
\end{figure}

We emphasize that the FMCI distributions are extremely useful in terms of summarizing the most important properties of the hidden Markov chain. We observe from Figure~\ref{FetalLambFMCIFig}a that the total number of runs in state~2 is between 1-5. Two runs are inferred from Viterbi and Posterior decoding, but this number of runs only amounts to 50\% of the posterior probability mass. Similarly, Viterbi and Posterior decoding infers 7 intervals in state~2, but the probability of more than 10 intervals in state~2 is larger than 15\% (Figure~\ref{FetalLambFMCIFig}b). Furthermore, Viterbi and Posterior decoding infers two regions in state~2. One region of length~1 and one region of length~6. FMCI is much more nuanced. From Figure~\ref{FetalLambFMCIFig}c we see that in general run lengths of size~1-7 are likely to occur. In general FMCI distributions are useful because they summarize the full posterior distribution of the hidden Markov chain.     
\subsection{Hybrid decoding and major yearly earthquakes, and the Artemis plot} \label{sec:ArtemisMotivate}
Hybrid decoding was originally formulated in \cite{LemberKolydenko2014}. Denote the hidden Markov chain by $y=(y_1,\ldots,y_n)$ and the observed sequence by $x=(x_1,\ldots,x_n)$. Posterior decoding minimizes the total number of expected pointwise errors by estimating which state is most likely for each position along the sequence 
\begin{eqnarray}
  {\rm Posterior}\;{\rm decoding:} \;\; s_t=\argmax_{u_t} \Prob(y_t=u_t|x), \;\; t=1,\ldots,n.
  \label{PosteriorDecoding}
\end{eqnarray}
Viterbi maximises the conditional probability of the hidden state sequence given the observed sequence 
\begin{eqnarray}
  {\rm Viterbi}\; {\rm decoding:} \;\; s=\argmax_u \Prob(y=u|x),
  \label{ViterbiDecoding}
\end{eqnarray}
where $u=(u_1,\ldots,u_n)$ is a hidden state sequence. Hybrid decoding is the sequence of hidden states $s=(s_1,\ldots,s_n)$ that maximizes the weighted and log-transformed loss functions from Posterior and Viterbi decoding. Mathematically we have
\begin{eqnarray}
  {\rm Hybrid}\;{\rm decoding:} \;\;
  s=
  \argmax_u 
  \Big\{
  (1-\alpha) \big[ \; \sum_{t=1}^{n} \log \Prob(y_t=u_t|x) \; \big]+
  \alpha \log \Prob(y=u|x) 
  \Big\},
  \label{HybridDecoding}
\end{eqnarray}
where $\alpha \in [0,1]$ is a tuning parameter. For $\alpha = 0$, hybrid decoding is equivalent to Posterior decoding, and for $\alpha = 1$, hybrid decoding is equivalent to Viterbi decoding. \cite{LemberKolydenko2014} and \cite{KuljusLember2023}) provide a dynamic programming algorithm, similar in spirit to the Viterbi algorithm, for identifying the hybrid path that maximizes the hybrid loss function 
\begin{eqnarray*}
  h(u)=
  (1-\alpha) \big[ \; \sum_{t=1}^{n} \log \Prob(y_t=u_t|x) \; \big]+
  \alpha \log \Prob(y=u|x) 
\end{eqnarray*}
as a function of $u=(u_1,\ldots,u_n)$. In Section~\ref{sec:geometric-mean} we discuss the hybrid loss function in more detail from the point of view of a weighted geometric mean. 

We note that hybrid decoding is trivial for the fetal movement count data from Section~\ref{sec:FMCImotivate}. In this case Viterbi and Posterior decoding are identical, and therefore hybrid decoding will also be identical to these two cases for any choice of tuning parameter~$\alpha$.   

Now consider the annual counts of major earthquakes in the years 1900 to 2006 from \cite{Zucchini2016} shown in Figure~\ref{fig:hybrid-earthquakes-simulated}A. In \cite{Zucchini2016}, the model parameters for a 2-state Poisson HMM are given by 
\begin{equation*}
    \pi = (1, 0),
    \;\; 
    \Gamma =
    \begin{pmatrix}
        0.928 & 0.072 \\
        0.119 & 0.881
    \end{pmatrix},
    \;\;
    {\rm and} \;\;
    \lambda = (15.4, 26.0).
\end{equation*}
When decoding the earthquakes data, we find that Viterbi and Posterior decoding differ in the years 1918 and 1973. When using hybrid decoding, we start by decoding the same path as Posterior decoding ($\alpha = 0$), which eventually becomes Viterbi ($\alpha = 1$) when increasing $\alpha$ (Figure~\ref{fig:hybrid-earthquakes-simulated}A). In this example, only one single unique hybrid path is available different from Posterior decoding and Viterbi: the hybrid path where $0.11\leq \alpha < 0.52$.

\begin{figure}[!htb]
    \centering
    \includegraphics[scale=0.4]{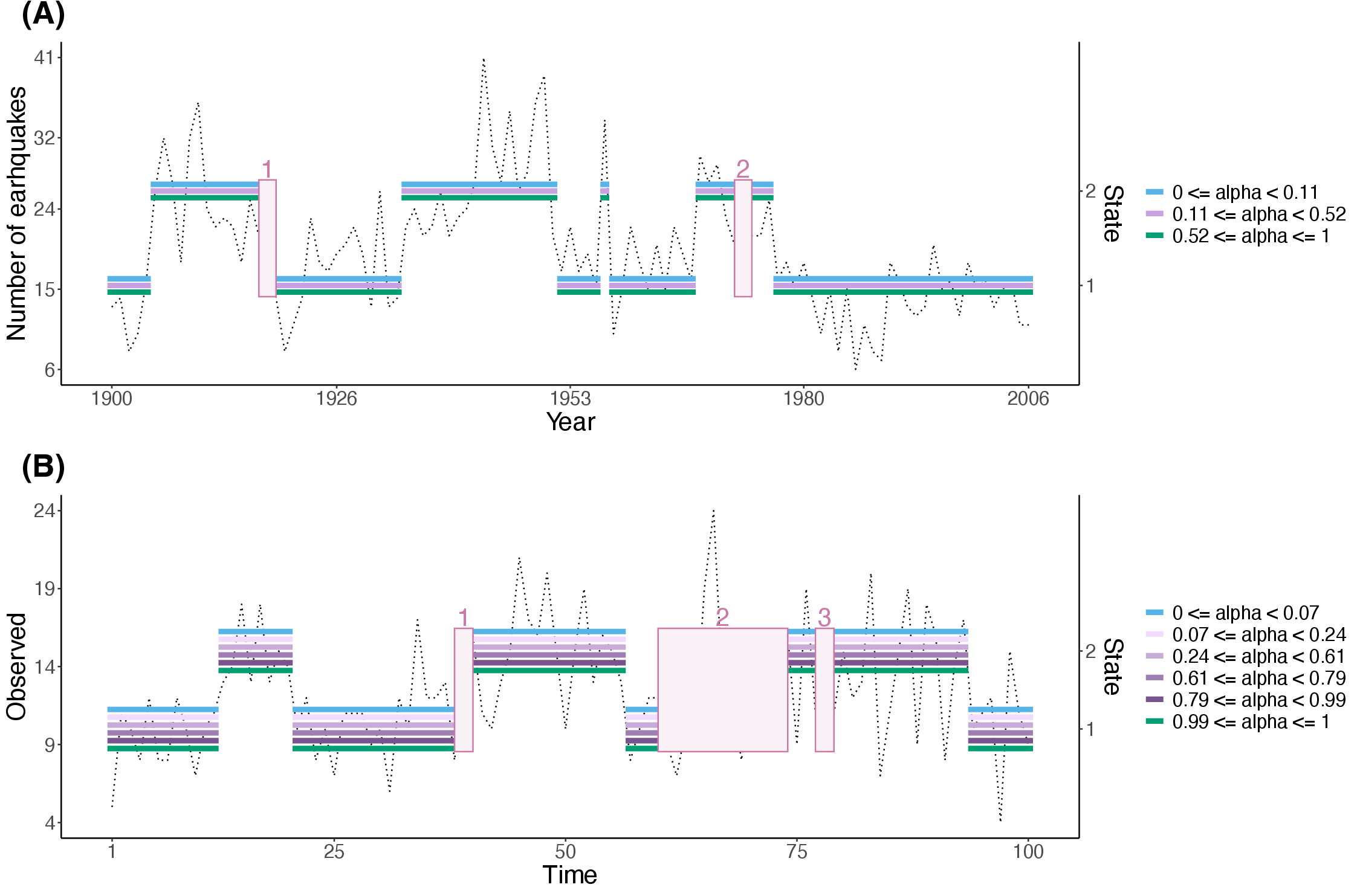}
    \caption{A visualization of how the hybrid paths change depending on~$\alpha$ for the earthquakes data~{\bf (A)} and a simulated example~{\bf (B)}. The dashed line is the observed sequence. {\bf (A)}: Posterior decoding and Viterbi differ in two years; 1918 and 1973 (pink box~1 and~2). For $0 \leq \alpha < 0.11$ the hybrid paths are identical to Posterior decoding. For $0.11 \leq \alpha < 0.52$ the hybrid paths are identical and different from both Posterior decoding and Viterbi, and for $0.52 \leq \alpha \leq 1$ the hybrid paths are identical to Viterbi. {\bf (B)}: Posterior decoding and Viterbi differ at three time points (pink boxes 1-3). For $0 \leq \alpha < 0.07$, the hybrid paths are identical to Posterior decoding. For $0.07 \leq \alpha < 0.99$, the hybrid paths differ from both Posterior decoding and Viterbi, and for $0.99 \leq \alpha \leq 1$, the hybrid paths are identical to Viterbi.}
    \label{fig:hybrid-earthquakes-simulated}
\end{figure}

In Figure~\ref{fig:hybrid-earthquakes-simulated}B we show an example where Posterior and Viterbi decoding differ in more positions. In this situation we simulated an observed sequence of length $n=100$ from a 2-state Poisson HMM given by
\begin{equation*}
    \pi = (1, 0), \;\;
    \Gamma =
    \begin{pmatrix}
        0.92 & 0.08 \\
        0.12 & 0.88
    \end{pmatrix},
    \;\; {\rm and} \;\;
    \lambda = (10, 15).
\end{equation*}
For this simulated example, there are four unique hybrid paths available between Posterior decoding and Viterbi, which differ in three regions (Figure~\ref{fig:hybrid-earthquakes-simulated}B). This example illustrates how different the Posterior decoding and Viterbi paths can be, as we see that Posterior decoding estimates short, but frequent, segments (pink boxes 2 and 3), whereas Viterbi is much more conservative and stays in the same state for a longer period of time. Importantly, we clearly see how the hybrid paths change their decoding to be more similar to Posterior decoding for values of $\alpha$ close to zero, and more similar to Viterbi for $\alpha$ close to one.

In large-scale applications with a large number of observations~$n$ it is impossible to visualize all the hybrid paths as in Figure~\ref{fig:hybrid-earthquakes-simulated}, and this is the background for our Artemis plot shown in Figure~\ref{fig:BasicArtemis}.
\begin{figure}[!ht]
    \centering
    \includegraphics[scale=0.5]{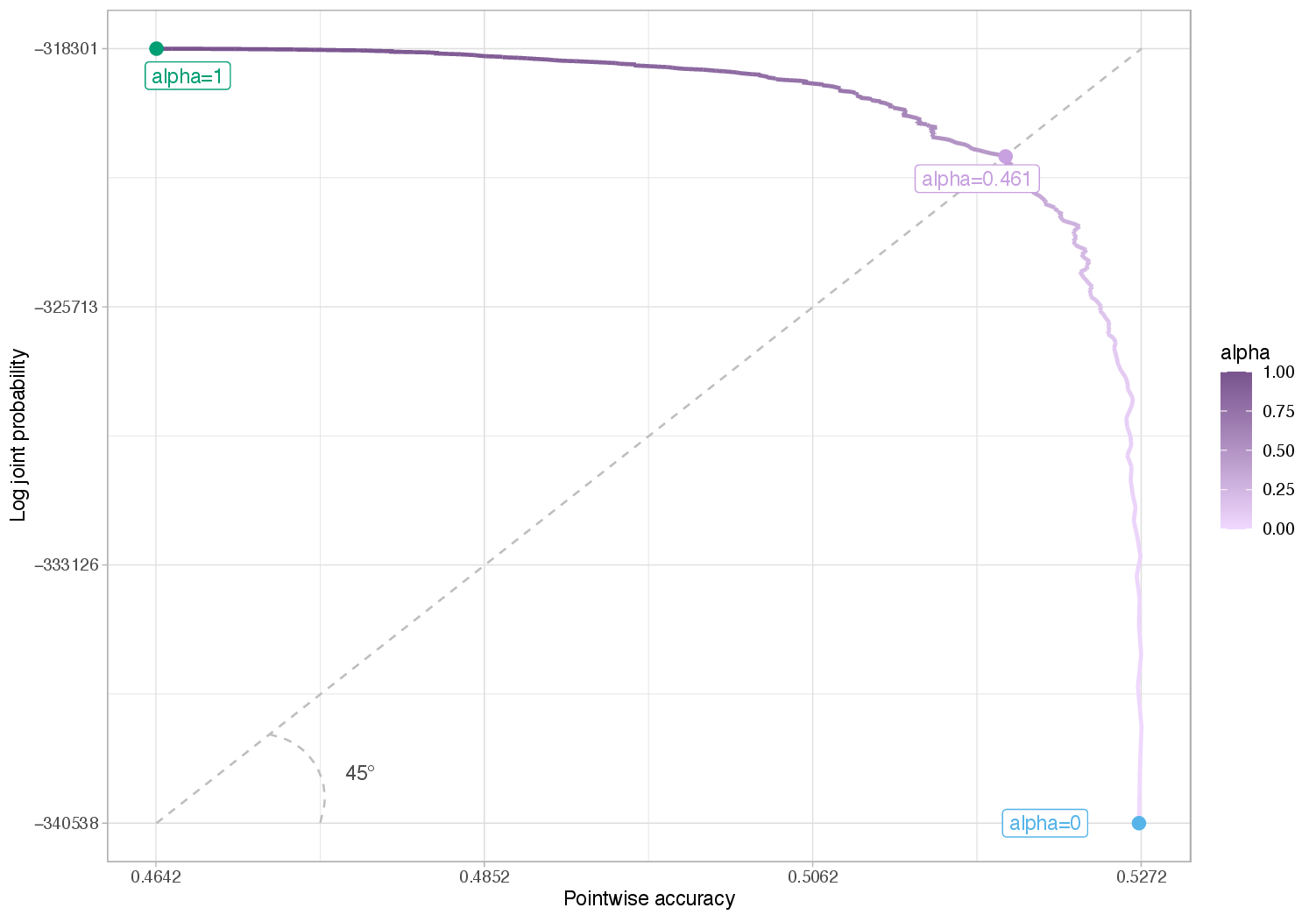}
    \caption{An illustration of an Artemis plot. On the x-axis is the pointwise accuracy, and on the y-axis is the log-joint probability. We start in the lower right corner with Posterior decoding ($\alpha = 0$), and increase $\alpha$ until we reach Viterbi ($\alpha = 1$). We choose the tuning parameter $\alpha$ as the value that intersects the bow-shaped curve at a 45 degree angle, as illustrated in the figure. Here $\alpha=0.461$ in the intersection point between the bow and the arrow.}
    \label{fig:BasicArtemis}
\end{figure}

The Artemis plot is constructed from a simulation study with a large sequence length of $n=10^5$. The model is a 3-state Poisson HMM with parameters
\begin{equation*}
    \pi = (0.8, 0.1, 0.1), \;\;
    \Gamma =
    \begin{pmatrix}
        0.8 & 0.1 & 0.1 \\
        0.1 & 0.8 & 0.1 \\
        0.1 & 0.1 & 0.8
    \end{pmatrix},
    \;\; {\rm and} \;\; 
    \lambda = (18, 20, 22).
\end{equation*}
Denote the simulated hidden state sequence by~$y$, the observed sequence by~$x$, and the decoded hidden state sequence by~$s$. Then pointwise accuracy is given by $\sum_{t=1}^n {\bf 1}(s_t=y_t)/n$ and the conditional probability is given by $\Prob(s|x)$. Instead of working with the conditional probablity we use the joint probability $\Prob(s,x)=\Prob(s|x)\Prob(x)$ because they only differ in terms of the multiplicative constant $\Prob(x)$ (the likelihood). In Figur~\ref{fig:BasicArtemis} we plot the joint values of the pointwise accuracy (x-axis) and log-joint probability (y-axis) for the hybrid paths with $\alpha$ varying from 0 (Posterior decoding) to 1 (Viterbi decoding). This is the Artemis plot.

Figure~\ref{fig:BasicArtemis} demonstrates that hybrid decoding is a useful construct because the joint values of accuracy and log-joint probability is bow-shaped as a function of the tuning parameter~$\alpha$. On the one hand, for small $\alpha$, we can achieve a decoding with an accuracy almost as high as for Posterior decoding ($\alpha=0$), but with a much higher conditional probability. On the other hand, for large $\alpha$, we can achieve a decoding with a conditional probability almost as high as for Viterbi decoding ($\alpha=1$), but with a much higher accuracy. In all of our simulations we have seen this bow-shaped curve for accuracy and conditional probability and varying $\alpha$ from~0 to~1. In general we suggest to choose the tuning parameter $\alpha$ in hybrid decoding by shooting out an arrow at a 45 degree angle, identify the point where it hits the bow, and report the corresponding value of $\alpha$, as illustrated in Figur~\ref{fig:BasicArtemis}. We call this figure an Artemis plot because Artemis in the ancient Greek religion always carried a bow and an arrow for hunting wild animals in the woods.    
\section{Summary statistics for the posterior hidden state sequence: Finite Markov chain imbedding (FMCI)} \label{Sec:FMCI}
In applications of hidden Markov models (HMMs), the main interest is often to compute the distribution of appropriate statistics. In this paper we focus on the the number of transitions between two hidden states, the time spent in a hidden state, the run length in a hidden state, and the longest hidden state run. In this section we review and apply the FMCI framework from \cite{AstonMartin2007} for obtaining pattern distributions of the posterior of the hidden state sequence for these four summary statistics. We also emphasize the role of sampling from the posterior and visualizations as important preliminary tools for setting up the FMCI framework with matrices of appropriate dimensions. This aspect of FMCI was only touched briefly in the discussion in \cite{AstonMartin2007}.   
\subsection{The posterior hidden state sequence is an inhomogeneous Markov chain}
\cite{AstonMartin2007} refers to \cite{Lindgren1978} for the result that the posterior hidden state sequence is an inhomogeneous Markov chain. For convenience we prove the result in this subsection. In Figure~\ref{Fig:iHMM} we provide a graphical proof of the Markov property. Note that the Markov property of the posterior hidden state sequence allows us to easily sample from the posterior distribution of the hidden state sequence (recall Figure~\ref{FetalLambFig}).
\begin{figure}[!htb]
  \centering
  \includegraphics[scale=0.7]{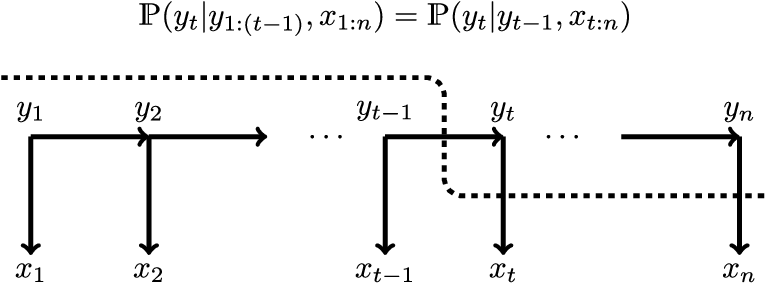}
  \caption{
  Graphical illustration of the result that the hidden state sequence 
  $(y_t)_{1\leq t \leq n}$, given the data $x=(x_t)_{1 \leq t \leq n}$, is an inhomogeneous Markov chain. The dashed line illustrates the variables that we condition upon in the left hand side in the equation.}
  \label{Fig:iHMM}  
\end{figure}

We let $\pi$ be the initial distribution, $\Gamma$ the transition probability matrix, and $\Phi$ the distribution of emissions for the HMM. The hidden state sequence is denoted $y=(y_1,\ldots,y_n)$ and the observed sequence is $x=(x_1,\ldots,x_n)$.
 
\begin{theorem}
    \label{thm:inhomo}
    In a hidden Markov model $(\pi,\Gamma,\Phi)$, the hidden state sequence $y=(y_1,\ldots,y_n)$ conditional on the observed sequence $x=(x_1,\ldots,x_n)$, is an inhomogeneous first-order Markov chain with transition probabilities
    \begin{eqnarray*}
       \Prob(y_t|y_{t-1},y_{t-2},\ldots ,y_1,x)
      = \Prob(y_t|y_{t-1},x_t,\ldots ,x_n) 
      = \frac{\beta_t(y_t)\Gamma_{y_{t-1},y_t}\Phi(x_t|y_t)}{\beta_{t-1}(y_{t-1})},
    \end{eqnarray*}
    for $t=2,\ldots,n$, where $\beta_t(y_t)=\Prob(x_{t+1},\ldots,x_n|y_t)$ is the matrix of backward probabilities. The conditional initial state probabilities are given by
    \begin{eqnarray}
        \Prob(y_1=i|x)=\frac{\beta_1(i)\pi_{i}\Phi(x_1|i)}{\Prob(x)}, \label{eq:Theorem1init}
    \end{eqnarray}
    where the denominator (the likelihood) is given by
    \begin{eqnarray*}
       \Prob(x)=\sum_{i} \beta_1(i) \pi_i \Phi(x_1|y_1=i).   
    \end{eqnarray*}
    Here, the last sum is over all possible hidden states.
\end{theorem}
\begin{proof}
    The result follows from a direct calculation.
    We have
    \begin{eqnarray*}
      \Prob(y_t|y_{1:(t-1)},x) 
      &=&
      \Prob(y_t|y_{t-1},x_{t:n}) 
      =
      \frac{\Prob(y_t,y_{t-1},x_{t:n})}
      {\Prob(y_{t-1},x_{t:n})} 
      =
      \frac{\Prob(x_{t:n}|y_t,y_{t-1})\Prob(y_t,y_{t-1})}
      {\Prob(x_{t:n}|y_{t-1})\Prob(y_{t-1})} \\
      &=&
      \frac{\beta_t(y_t)\Phi(x_t|y_t)\Gamma_{y_{t-1},y_t}}
      {\beta_{t-1}(y_{t-1})}.
    \end{eqnarray*}
    The first equality is shown in Figure~\ref{Fig:iHMM} and in the last equality we used
    $\Prob(y_t,y_{t-1})/\Prob(y_{t-1})=\Gamma_{y_{t-1},y_t}$ and 
    \begin{eqnarray*}
      \Prob(x_{t:n}|y_t,y_{t-1}) = \Prob(x_{t:n}|y_t)=
      \Prob(x_t|y_t)\Prob(x_{(t+1):n}|y_t)=
      \Prob(x_t|y_t)\beta_t(y_t)=
      \Phi(x_t|y_t) \beta_t(y_t).
    \end{eqnarray*}
\end{proof}   
\noindent A huge advantage of Theorem~\ref{thm:inhomo} is that we can sample from the posterior hidden state sequence and subsequently use the samples to summarize various properties of the posterior for the hidden state sequence. This is exactly what we did in Section~\ref{sec:FMCImotivate} for the fetal lamb movements.

In the next subsection we describe how to avoid sampling by taking advantage of so-called finite Markov chain imbedding (FMCI). We emphasize, however, that the dimension of the transition matrices in FMCI are often informed from the empirical distributions from the samples.
\subsection{Summary statistics for inhomogeneous Markov chains: \\
Aston-Martin framework for the posterior hidden state sequence} 
The Aston-Martin framework apply the FMCI methodology originally developed in \cite{FuKoutras1994} to calculate distributions of patterns associated with the Bernoulli trials. Here, we illustrate FMCI using a HMM with two hidden states. We consider four types of summary statistics for the hidden states: (i) The number of jumps $J$ from state~1 to state~2, (ii) The number of positions $N$ in state~2, (iii) The exact run length $E_k$ of size $k$ in state~2, and (iv) The longest run $L$ of in state~2. 
\subsubsection{Number of jumps $J$}
Suppose we are interested in the number of transitions from state~1 to state~2. Theorem \ref{thm:inhomo} gives us, at any time point $t$, the probability~$a_t$ of staying in state~1 and the probability~$b_t$ of staying in state~2. Now consider the FMCI of the situation depicted in Figure~\ref{JumpsFMCIFig}.   
\begin{figure}[!htb]
  \centering
  \includegraphics[scale=0.45]{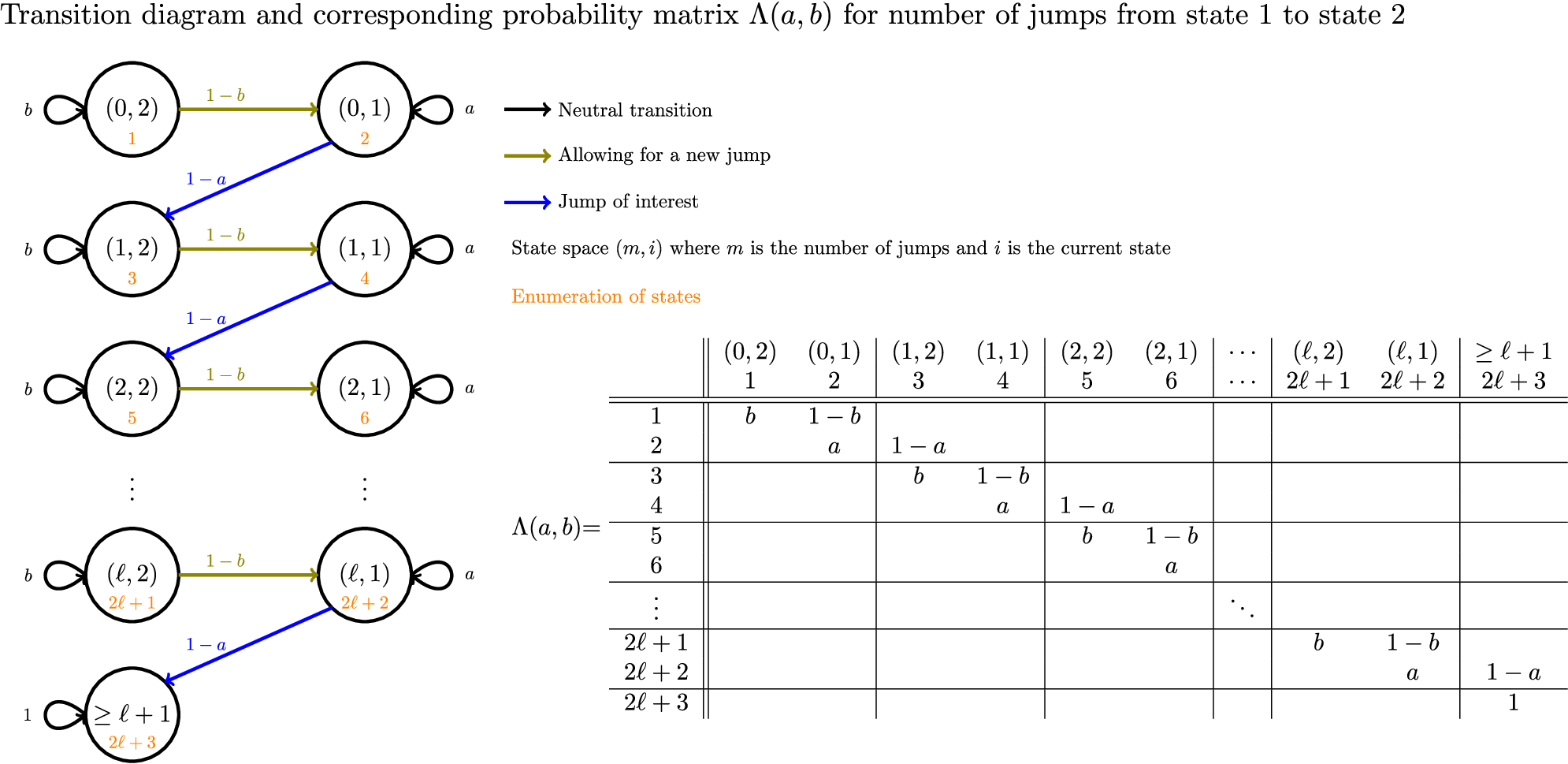}
  \caption{
  Illustration of finite Markov chain imbedding for the number of jumps from state~1 to state~2. The parameters $a$ and $b$ are the probabilities of staying in state~1  and state~2, respectively, at time~$t$, i.e. 
  $a_t=\Prob(y_t=1|y_{t-1}=1,x)$ and $b_t=\Prob(y_t=2|y_{t-1}=2,x),$ for $t=2,\ldots,n$. The index for the time~$t$ is suppressed in the figure.
  } 
  \label{JumpsFMCIFig}
\end{figure}

Each state in Figure~\ref{JumpsFMCIFig} has two indices. The first index is the number of $1\rightarrow 2$ transitions up to this point and the second index is the current state. The transition probabilities follow from the interpretation of $a$ and $b$ (here, the time $t$ is suppressed). For example, if the current state is $(2,2)$ it means that two $1\rightarrow 2$ transitions have already occured and the current state is state~2. Thus, the next state is $(2,2)$ with probability $b$ (staying in state~2) and $(2,1)$ with probability $1-b$ (transition from state~2 to state~1 and still two $1\rightarrow 2$ transitions have occured). If the current state is $(2,1)$ it means that two $1\rightarrow 2$ transitions have already occured, and the current state is state~1. Thus, the next state is $(2,1)$ with probability~$a$ (staying in state~1) and $(3,2)$ with probability $1-a$ (a third transition from state~1 to state~2 has occured, and the new state is state~2). We collect the transition probabilities at time~$t$ in the matrix $\Lambda(a_t,b_t)$, where the transition probability matrix is provided in Figure~\ref{JumpsFMCIFig}. The starting state in the FMCI process is $(0,2)$ with probability $\eta_2$ and $(0,1)$ with probability $\eta_1$, where $\eta_i=\Prob(y_1=i|x), \; i=1,2,$ are the initial state probabilities given in equation~(\ref{eq:Theorem1init}) in Theorem~\ref{thm:inhomo}.

Let $\eta=(\eta_2,\eta_1,0,\ldots,0)$ be the vector of length $2\ell+3$ of initial probabilities. The posterior distribution for the number of jumps is now determined by the vector 
\begin{eqnarray}
  \eta \prod_{t=2}^n \Lambda(a_t,b_t),
  \label{FMCIeq}
\end{eqnarray}
which is of length $2\ell+3$. Here, the sum of the two first entries is the probability of zero jumps, the sum of entries three and four is the probability of one jump, and so on. The last entry is the probability of $\ell+1$ or more jumps. It is natural to simulate from the inhomogeneous Markov chain and use summaries from the simulations to inform about the size $2\ell+3$ of the transition matrix. For the fetal lamb movement data in Section~\ref{sec:FMCImotivate} we choose $\ell=7$ because the empirical distribution in Figure~\ref{FetalLambFMCIFig}a shows that more than 8 transitions between state~1 and state~2 are unlikely for this data.

We end this subsection by emphasizing the block structure and sparsity of the transition matrix. Defining 
the block
\begin{eqnarray*}
  B(a,b)=
  \begin{pmatrix}
    \begin{tabular}{cc}
      $b$ & $1-b$ \\ 
      $0$ & $a$
    \end{tabular}
  \end{pmatrix}
\end{eqnarray*}
we that have $\Lambda=\Lambda(a,b)$ is given by
\begin{eqnarray*}
  \begin{tabular}{c||cccccccccc}
    & 1 & 2 & 3 & 4 & $\cdots$ & $2\ell+1$ & $2\ell+2$ & $2\ell+3$ \\ \hline \hline 
    1 & \multicolumn{2}{c|}{\multirow{2}{*}{$B(a,b)$}} & \multicolumn{2}{c|}{} && \\ 
    2 & \multicolumn{2}{c|}{} & $1-a$ & \multicolumn{1}{c|}{0}\\ \cline{1-5}
    3 & && \multicolumn{2}{|c|}{\multirow{2}{*}{$B(a,b)$}} \\ 
    4 & && \multicolumn{2}{|c|}{} &  \\ \cline{4-7} 
    $\vdots$ & &&&& \multicolumn{2}{|c|}{$\ddots$} \\ \cline{6-8}
    $2\ell+1$ && && && \multicolumn{2}{|c|}{\multirow{2}{*}{$B(a,b)$}} \\
    $2\ell+2$ && && && \multicolumn{2}{|c|}{} & $1-a$ \\ \cline{7-9}
    $2\ell+3$ && && && && \multicolumn{1}{|c}{1} 
  \end{tabular}.
\end{eqnarray*}
The other three summary statistics also exhibit block structure and sparsity. This feature greatly simplify implementation of FMCI.

If instead of the number of jumps from state~1 to state~2 we are interested in the number of runs in state~2, then we should use $\eta=(0,\eta_1,\eta_2,0,\ldots,0)$ as the initial probability vector because beginning in state~2 in the first position starts a run.
\subsubsection{Number of positions $N$}
In Figure~\ref{NFMCIFig} we show the transition matrix for the total number of positions in the state~2. The interpretation of the states in in the Markov imbedding is very similar to the number of jumps except that now we count every time we enter state~2. 
\begin{figure}[htb!]
  \centering
  \includegraphics[scale=0.5]{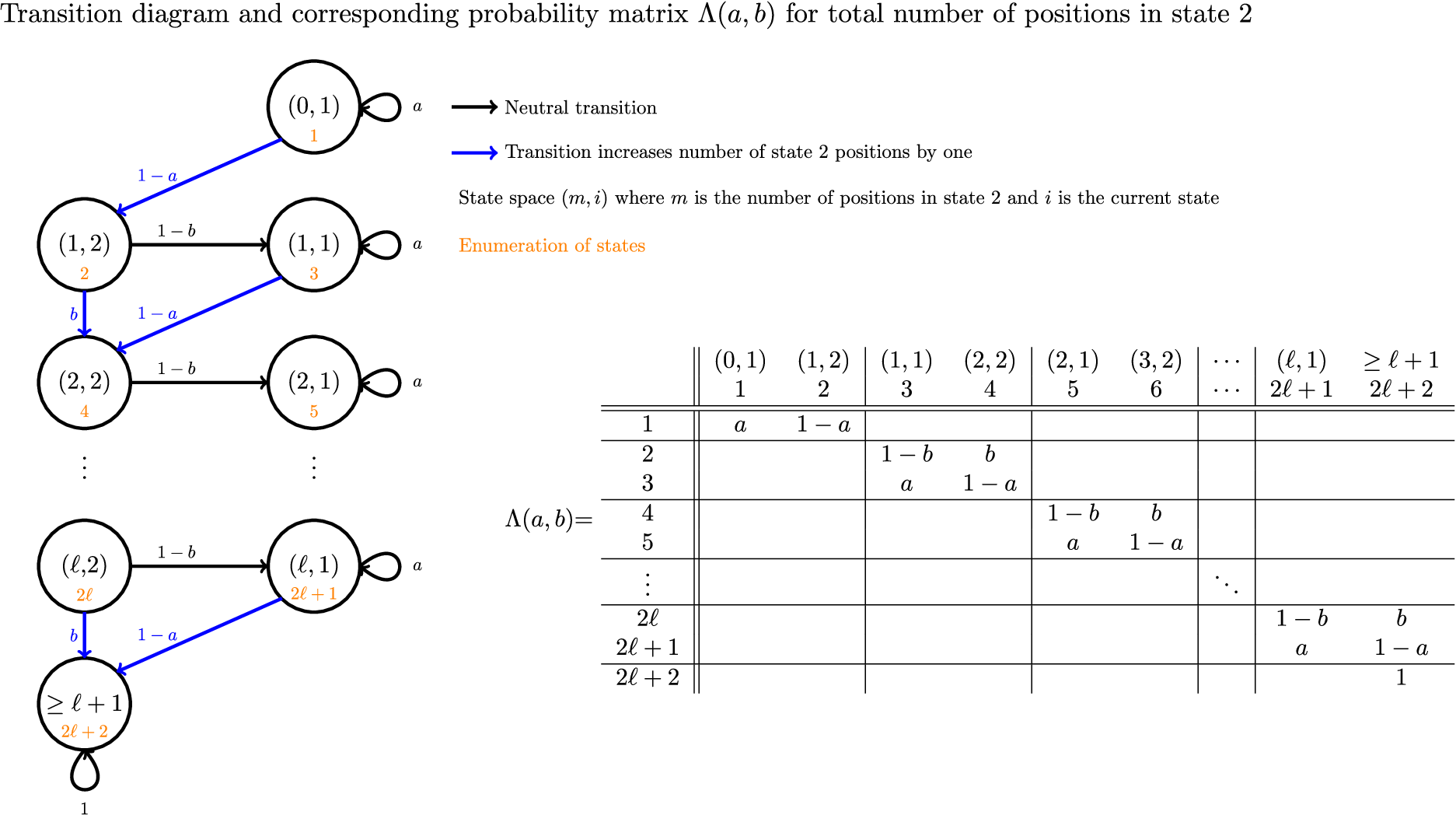}
  \caption{
  Illustration of finite Markov chain imbedding for the total number of positions in state~2. The parameters $a$ and $b$ are the probabilities of staying in state~1 and and state~2, respectively, at time~$t$, i.e. $a_t=\Prob(y_t=1|y_{t-1}=1,x)$ and $b_t=\Prob(y_t=2|y_{t-1}=2,x),\;\; t=2,\ldots,n$. The index for the time~$t$ is suppressed in the figure.
  } 
  \label{NFMCIFig}
\end{figure}

We define the block matrix by
\begin{eqnarray*}
  B=B(a,b)=
  \begin{pmatrix}
  \begin{tabular}{cc}
    $1-b$ & $a$ \\ 
    $a$ & $1-a$
  \end{tabular}
  \end{pmatrix},
\end{eqnarray*}
and the corresponding transition matrix $\Lambda=\Lambda(a,b)$ can be written in block format as
\begin{eqnarray*}
  \begin{tabular}{c||ccccccccc}
  & 1 & 2 & 3 & 4 & 5 & 6 & $\cdots$ & $2\ell+1$ & $2\ell+2$ \\ \hline \hline
  1 & $a$ & \multicolumn{1}{c|}{$1-a$} \\ \cline{1-5}
  2 &&& \multicolumn{2}{|c|}{\multirow{2}{*}{$B(a,b)$}}  \\ 
  3 &&& \multicolumn{2}{|c|}{} \\ \cline{1-1}\cline{4-7} 
  4 &&&&& \multicolumn{2}{|c|}{\multirow{2}{*}{$B(a,b)$}} \\ 
  5 &&&&& \multicolumn{2}{|c|}{}  \\ \cline{1-1}\cline{6-8} 
  $\vdots$ &&&&&&& \multicolumn{1}{|c|}{$\ddots$} \\ \cline{1-1}\cline{8-10}
  $2\ell$ &&&&&&&& \multicolumn{2}{|c|}{\multirow{2}{*}{$B(a,b)$}} \\
  $2\ell+1$ &&&&&&&& \multicolumn{2}{|c|}{}  \\ \cline{1-1}\cline{9-10}
  $2\ell+2$ &&&&&&&&& \multicolumn{1}{|c|}{1} 
  \end{tabular}
\end{eqnarray*}
The beginning state is $\eta=(\eta_1,\eta_2,0,\ldots,0)$. 

The posterior distribution for the number of positions in the state~2 can now again be determined by the product~(\ref{FMCIeq}) of the beginning vector and the the transition probablility matrices at each position in the Markov chain. Here, the first entry is the probability of zero positions in state~2 in the hidden Markov chain, the sum of entry two and three is the probability of one position in state~2, and so on. The last entry $2\ell+2$ is the probability of $\ell+1$ or more positions in state~2 in the hidden Markov chain. In Figure~\ref{FetalLambFMCIFig}b we show the distribution for the number of state~2 positions for the fetal lamb movement data. Again, the choice of the size of the transition probability matrix was informed by the simulation study.   
\subsubsection{Exact run length $E_k$}
In Figure~\ref{EFMCIFig} we show the transition matrix from finite Markov chain imbedding for determining the exact run length $E_k$ in state~2 for any run length $k \geq 1$.  
\begin{figure}[htb!]
  \centering
  \includegraphics[scale=0.41]{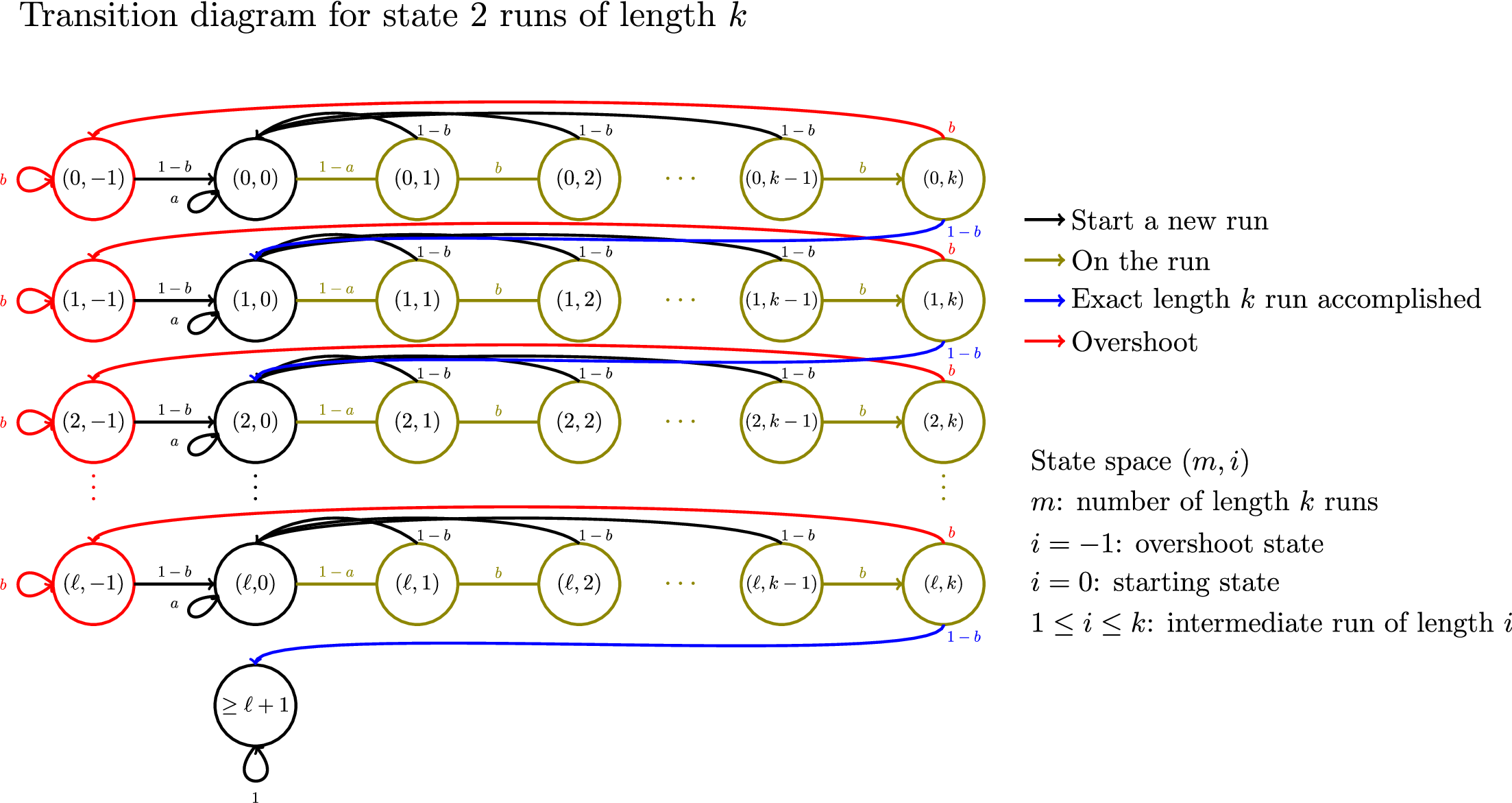}
  \caption{
  Illustration of finite Markov chain imbedding for exact run length $E_k$ of length $k$ in state~2. The parameters $a$ and $b$ are the probabilities of staying in state~1 and state~2, respectively, at time~$t$, i.e. $a_t=\Prob(y_t=1|y_{t-1}=1,x),$ and $b_t=\Prob(y_t=2|y_{t-1}=2,x),\;\; t=2,\ldots,n$. The index for the time~$t$ is suppressed in the figure.
  } 
  \label{EFMCIFig}
\end{figure}

From the figure we see a clear block structure. Indeed, we define the $(k+2)\times(k+2)$ block matrix
\begin{eqnarray*}
  B_k(a,b)=
  \begin{tabular}{c||cc|ccccccc}
  & 1 & 2 & 3 & 4 & $\cdots$ & $k+1$ & $k+2$ \\ \hline \hline
  1 & $b$ & $1-b$ \\ 
  2 & & $a$ & $1-a$ \\ \hline 
  3 && $1-b$ & 0 & $b$ \\   
  $\vdots$ && $\vdots$ &&& $\ddots$ & $\ddots$ \\ 
  $k+1$ && $1-b$ &&&& $0$ & $b$ \\
  $k+2$ & $b$& $0$ &&&&& $0$ 
\end{tabular}.
\end{eqnarray*}
The first state in the block is an overshoot state (the current run is larger than the desired length~$k$), the second state in the block erases all previous runs (the hidden Markov chain is in state~1), and in the third state in the block a new run of state~2's in the hidden Markov chain has started. The fate of this new run has three possibilities: it has a premature state~1 (and goes back to block state~2), it overshoots (and goes back to block state~1) or a run of exactly length $k$ is accomplished in which case the FMCI chain goes to the second state in the next block.
The transition probability matrix can be written in terms of the block matrices as
\begin{eqnarray*}
  \Lambda=\Lambda(a,b)=
  \begin{pmatrix}
  \begin{tabular}{ccccccccccccc}
  \multicolumn{3}{c|}{\multirow{3}{*}{$B_k(a,b)$}} & \multicolumn{3}{c}{} && \\
  \multicolumn{3}{c|}{} & \multicolumn{3}{c}{} \\
  \multicolumn{3}{c|}{} & $0$ & $1-b$ & \multicolumn{1}{c}{} \\ \cline{1-6}
  &&& \multicolumn{3}{|c|}{\multirow{3}{*}{$B_k(a,b)$}} \\ 
  &&& \multicolumn{3}{|c|}{} & \multicolumn{3}{c}{}  \\ 
  &&& \multicolumn{3}{|c|}{} & $0$ & $1-b$ & \multicolumn{1}{c}{} \\ \cline{4-9} 
  &&& &&& \multicolumn{3}{|c|}{$\ddots$} \\ \cline{7-12}
  &&& &&& &&& \multicolumn{3}{|c|}{\multirow{3}{*}{$B_k(a,b)$}} \\
  &&& &&& &&& \multicolumn{3}{|c|}{} \\
  &&& &&& &&& \multicolumn{3}{|c|}{} & $1-b$ \\ \cline{10-13}
  &&& &&& &&& &&& \multicolumn{1}{|c}{1} 
  \end{tabular}
  \end{pmatrix}.
\end{eqnarray*}
The posterior distribution for the number of runs in state~2 of a certain length $k$ can now be determined by the product~(\ref{FMCIeq}) where the beginning state is $\eta=(0,\eta_1,\eta_2,0,\ldots,0)$. The sum of the first $(k+2)$ entries in the product is the probability of zero runs of length~$k$, the sum of the next $(k+2)$ entries is the probability of exactly one run of length~$k$, and so one. From these probabilities we can then calculate the expected number of runs of length~$k$. These expected values are reported in Figure~\ref{FetalLambFMCIFig}c.  
\subsubsection{Longest run length $L$}
The transition matrix for the longest run length~$L$ is visualized in~Figure~\ref{LFMCIFig}. 
\begin{figure}[htb!]
  \centering
  \includegraphics[scale=0.42]{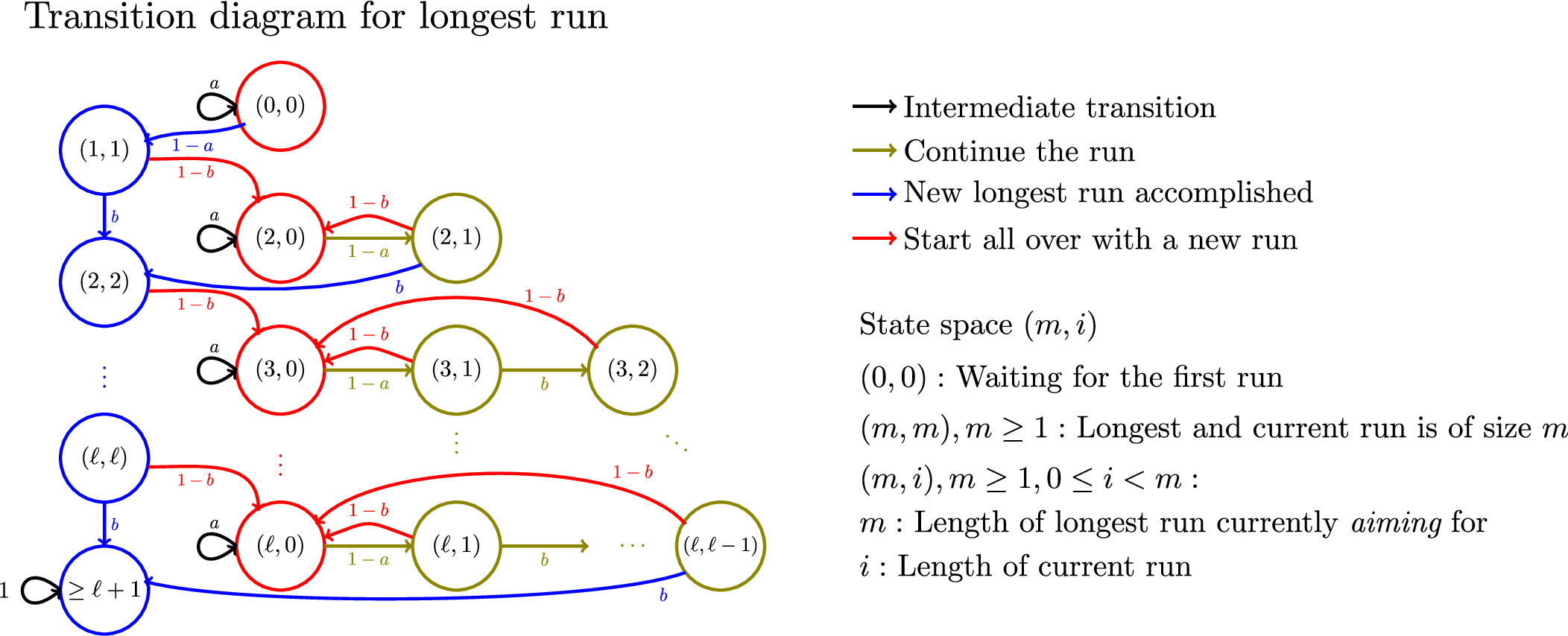}
  \caption{
  Illustration of finite Markov chain imbedding for longest run length $L$. The parameters $a$ and $b$ are the probabilities of staying in state~1 and state~2, respectively, at time~$t$, i.e. $a_t=\Prob(y_t=1|y_{t-1}=1,x)$ and $b_t=\Prob(y_t=2|y_{t-1}=2,x),\;\; t=2,\ldots,n$. The index for the time~$t$ is suppressed in the figure.
  } 
  \label{LFMCIFig}
\end{figure}

We again observe a clear block structure where the blocks increase in size with the length of the run. We define the $(k+2)\times(k+2)$ block matrices $B_k(a,b)$ for $k \geq 1$ by
\begin{eqnarray*}
  B_k(a,b)=
  \begin{tabular}{c||cc|ccccc}
    & 1 & 2 & 3 & $\cdots$ & $k+2$ \\ \hline \hline
    1 & $1-b$ & & \\ 
    2 & $a$ & $1-a$ \\  \hline
    3 & $1-b$ && $b$ \\   
    $\vdots$ &&&& $\ddots$ \\ 
    $k+2$ & $1-b$& &&& $b$ 
  \end{tabular}.
\end{eqnarray*}

The transition matrix is then given by 
\begin{eqnarray*}
  \Lambda=\Lambda(a,b)=
  \begin{pmatrix}
    \begin{tabular}{ccccccccccccc}
      $a$ & \multicolumn{1}{c|}{$1-a$} \\ \cline{1-3}
      \multicolumn{2}{c|}{} & \multicolumn{1}{c|}{$B_1$} & \\ \cline{3-4}
      &&& \multicolumn{1}{|c|}{$B_2$} \\ \cline{4-5} 
      &&& & \multicolumn{1}{|c|}{$\ddots$} \\ \cline{5-7}
      &&& && \multicolumn{2}{|c}{$B_{\ell}$} \\ \cline{6-7}
      &&& &&& \multicolumn{1}{c}{1} 
    \end{tabular}
  \end{pmatrix},
\end{eqnarray*}
and the initial distribution is given by 
\begin{eqnarray*}
  \eta=(\eta_1,\eta_2,0,\ldots,0).
\end{eqnarray*}

For the fetal lamb movement data, the distribution of the longest run is shown in Figure~\ref{FetalLambFMCIFig}d.
\section{Hybrid decoding} \label{Sec:HybridDecoding}
We begin this section by providing an alternative derivation of the hybrid method using a weighted geometric mean approach (Section~\ref{sec:geometric-mean}). We then suggest a novel procedure, based on the Artemis plot, for choosing the tuning parameter~$\alpha$ (Section~\ref{sec:choosingAlpha}). The last two investigations in this section on hybrid decoding are concerned with robustness of Artemis analysis across HMMs (Section~\ref{sec:Robustness}) and blockwise accuracy for hybrid decoding (Section~\ref{sec:Blockwise}). We find that Artemis plots have the desired bow-shape for a number of different HMMs, and we find in general that the hybrid method has higher accuracy than Posterior decoding and Viterbi in terms of decoding blocks of hidden states of intermediate sizes.        
\subsection{Hybrid risk as a weighted geometric mean} \label{sec:geometric-mean}
Recall from (\ref{PosteriorDecoding}) that posterior decoding minimizes the total number of expected pointwise errors by estimating which state is most likely for each position along the sequence 
\begin{eqnarray*}
  {\rm Posterior}\;{\rm decoding:} \;\; s_t=\argmax_{u_t} \Prob(y_t=u_t|x), \;\; t=1,\ldots,n.
\end{eqnarray*}
However, Posterior decoding does not take into account the overall conditional probability of the sequence $\Prob(s|x)$. On the other hand, Viterbi decoding (recall (\ref{ViterbiDecoding})) maximises the conditional probability of the sequence 
\begin{eqnarray*}
  {\rm Viterbi}\; {\rm decoding:} \;\; s=\argmax_u \Prob(y=u|x),
\end{eqnarray*}
which can result in a low pointwise accuracy. Hence, the strength of Posterior decoding is the limitation of Viterbi and vice versa; recall the Artemis plot in Figure~\ref{fig:BasicArtemis}. A risk-based framework for combining Posterior decoding and Viterbi into a common framework was introduced in \cite{LemberKolydenko2014} and \cite{KuljusLember2023}, and we provide a brief review of their method in the Appendix. 

Posterior decoding maximizes the geometric mean
\begin{eqnarray*}
  G(u)=\Big( \prod_{t=1}^n \Prob(y_t = u_t|x) \Big)^{1/n}
\end{eqnarray*}
with respect to the estimated hidden state sequence $u=(u_1,\ldots,u_n)$.
Let $V(u)=\Prob(y=u|x)$ be the probability of the hidden state sequence conditional on the data. Then Viterbi decoding maximizes~$V(u)$. 

Now consider the weighted geometric mean of Posterior and Viterbi decoding
\begin{eqnarray*}
  H_{\alpha}(u)=
  \Big( 
  G(u)^{w_1}V(u)^{w_2} 
  \Big)^{1/(w_1+w_2)},
\end{eqnarray*}
where $w_1=n(1-\alpha)$ and $w_2=n\alpha$ are the weights for the two means. 
Maximization of the weighted geometric mean is identical to the definition of hybrid decoding in~(\ref{HybridDecoding}) because 
\begin{eqnarray}
  w_1 \log G(u)+w_2 \log V(u)=
  (1-\alpha) \Big\{ \; \sum_{t=1}^{n} \log \Prob(y_t=u_t|x) \; \Big\}+
  \alpha \log \Prob(y=u|x).
  \label{HybridGeometricMean}
\end{eqnarray}
We finally note that maximization of the weighted mean~(\ref{HybridGeometricMean}) with respect to the hidden state sequence~$u$ is identical to minimizing the hybrid risk defined in (\ref{eq:hybrid-risk}).
\subsection{Artemis analysis and choice of tuning parameter $\alpha$} \label{sec:choosingAlpha}
To apply the hybrid method in practice, an appropriate value of the tuning parameter $\alpha \in [0,1]$ is required. In this section, we describe our procedure for choosing $\alpha$ based on the pointwise accuracy and joint probability of the hybrid path and observed sequence, which we refer to as Artemis analysis. The idea is simple but powerful: posterior decoding maximizes the pointwise accuracy, whereas Viterbi maximizes the joint probability, and hence it is natural to choose a hybrid path that borrows strength from both of these two desirable properties of the decoded hidden state sequence.

The Artemis analysis is a simulation-based procedure because in order to compute the accuracy we need to know the true hidden state sequence. The simulations are from the fitted HMM. In this section we consider a 3-state Poisson HMM with parameters
\begin{equation}
    \pi = (0.8, 0.1, 0.1), \;\;
    \Gamma =
    \begin{pmatrix}
        0.8 & 0.1 & 0.1 \\
        0.1 & 0.8 & 0.1 \\
        0.1 & 0.1 & 0.8
    \end{pmatrix}, 
    \;\; {\rm and} \;\;
    \lambda = (20-a, 20, 20+a).
    \label{Cases3statePoisson}
\end{equation}
The emissions are divided into three cases of increasing decoding difficulty: easy ($a = 10$), medium ($a = 5$), and hard ($a = 2$). Denote the simulated hidden Markov chain by $y=(y_1,\ldots,y_n)$ and the simulated observed sequence by $x=(x_1,\ldots,x_n)$. For each simulation~$(x,y)$ and selected values of $\alpha$, the hybrid path~$s=s(x,\alpha)$ is inferred, and the pointwise accuracy $\sum_{i=1}^n \mathbbm{1}(s_t=y_t)/n$ and log-joint probability $\log\Prob(y=s,x)$ is calculated. We calculate the log joint probability, since the conditional path probability $\Prob(s|x)$ (which Viterbi maximises), can be rewritten as $\Prob(s|x) = \Prob(s,x)/\Prob(x)$. Since $\Prob(x)$ is the likelihood of data, it is a constant, and hence maximising the conditional path probability $\Prob(s|x)$ is equivalent to maximising the joint probability $\Prob(s,x)$. The joint values of pointwise accuracy and log-joint probability are mapped in the Artemis plot as shown in Figure~\ref{fig:artemis_i_iii}. The optimal $\alpha$ for each simulation is identified as the $\alpha$ on the curve at an angle of 45 degrees (recall Figure~\ref{fig:BasicArtemis}). 
\begin{figure}[!htb]
    \centering
    \includegraphics[width=0.95\linewidth]{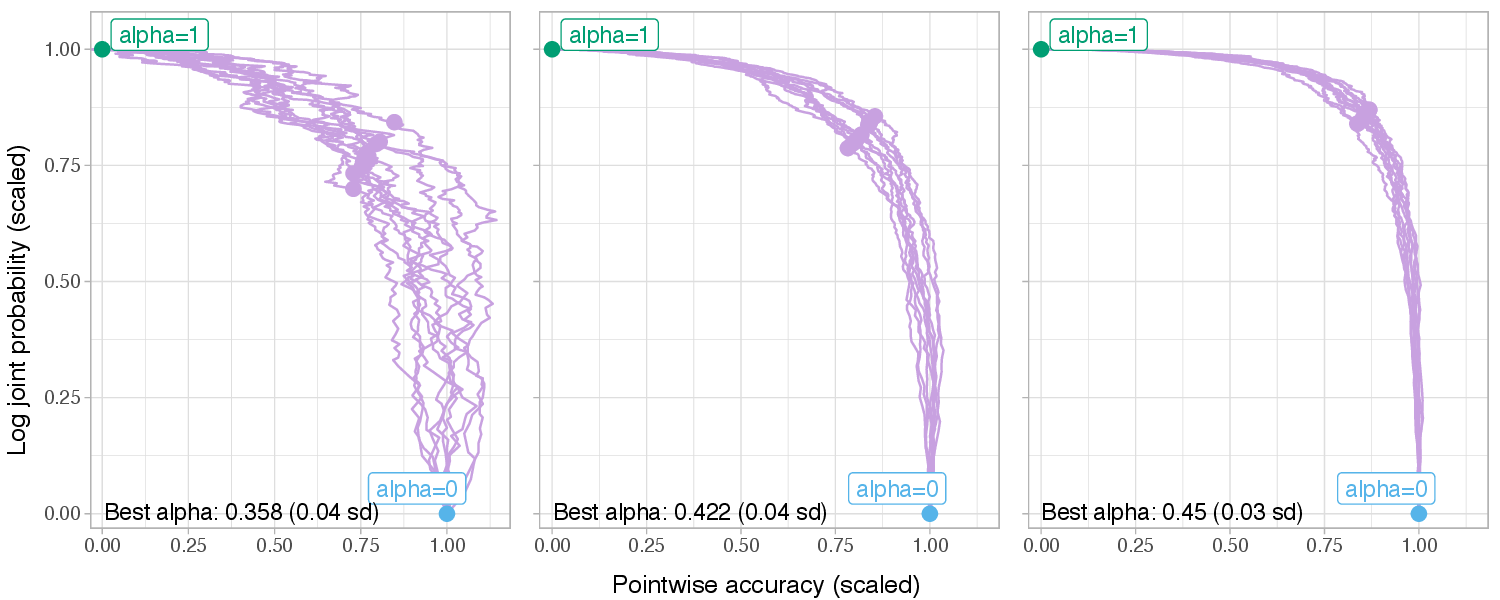}
    \caption{Artemis plots for the three cases of increasing decoding difficulty determined by the emission rates. The rates are given by $\lambda = (20-a, 20, 20+a)$, where $a = 10$ for the easy case, $a = 5$ for the medium case, and $a = 2$ for the hard case. The axes are scaled using the minimum and maximum value for each axis. Each line represents a simulated observed and hidden state sequence (10 for each case) and the sequence length is $n = 10^5$. The optimal $\alpha$ is the average of the 10 simulations for each case. We get an optimal value of 0.358, 0.422, and 0.45 for the 3-state Poisson models considered here.}
    \label{fig:artemis_i_iii}
\end{figure}

In Figure~\ref{fig:artemis_i_iii} we have used 10 simulations for each of the three cases and the simulated sequences are of length $n=10^5$ (we discuss the choice of simulated sequence length below).  For each case, the simulations follow the same pattern. For $\alpha = 0$ (Posterior decoding), the paths have a high pointwise accuracy and a low conditional path probability, and when we increase $\alpha$, we observe a distinct bow-shape, until we reach $\alpha = 1$ (Viterbi), where the paths have a low pointwise accuracy, but a high conditional path probability. Hence, choosing $\alpha$ at a 45 degree angle will result in a hybrid path that balances in a suitable way both the pointwise accuracy and the conditional path probability.

In Table~\ref{tab:distribution-alphas} we provide the values of the optimal~$\alpha$ for each of the simulations for each of the three cases. We recommend to use the average $\alpha$ from the simulations as the value of the tuning parameter~$\alpha$ in hybrid decoding. 
\begin{table}[!htb]
    \centering
    \begin{tabular}{c|ccc}
        \textbf{Iteration} & \textbf{Easy ($a=10$)} & \textbf{Medium ($a=5$)} & \textbf{Hard ($a=2$)} \\
        \hline
        1 & 0.371 & 0.449 & 0.461 \\
        2 & 0.344 & 0.480 & 0.453 \\
        3 & 0.414 & 0.395 & 0.430 \\
        4 & 0.289 & 0.391 & 0.461 \\
        5 & 0.336 & 0.414 & 0.410 \\
        6 & 0.340 & 0.480 & 0.461 \\
        7 & 0.383 & 0.418 & 0.477 \\
        8 & 0.312 & 0.363 & 0.480 \\
        9 & 0.438 & 0.383 & 0.457 \\
        10 & 0.352 & 0.449 & 0.406 \\ \hline 
        average & 0.358 & 0.422 & 0.450 \\
        std dev & 0.045 & 0.041 & 0.025
    \end{tabular}
    \caption{The values of optimal $\alpha$ for each of the 10 simulations for the three cases in Figure~\ref{fig:artemis_i_iii}. We recommend to use the average $\alpha$ from the simulations as the final value of the tuning parameter $\alpha$ in hybrid decoding.}
    \label{tab:distribution-alphas}
\end{table}
\subsubsection{Choice of simulated sequence length in Artemis plot}
Figure~\ref{fig:size-of-n} shows Artemis plots of four simulations with different values of sequence length $n$ from the medium case with $a = 5$. Choosing $\alpha$ based on simulations with $n = 100$ versus simulations with $n = 10^5$ leads to very different results. When $n$ is small, Posterior decoding and Viterbi will most often only differ in a few positions, and only very few different hybrid paths are available, which is the behavior we see in the upper left corner in Figure~\ref{fig:size-of-n}. The situation is similar to the behavior in Figure~\ref{fig:hybrid-earthquakes-simulated} where the sequence length was also around $n=100$. 
\begin{figure}[!htb]
    \centering
    \includegraphics[width=0.65\linewidth]{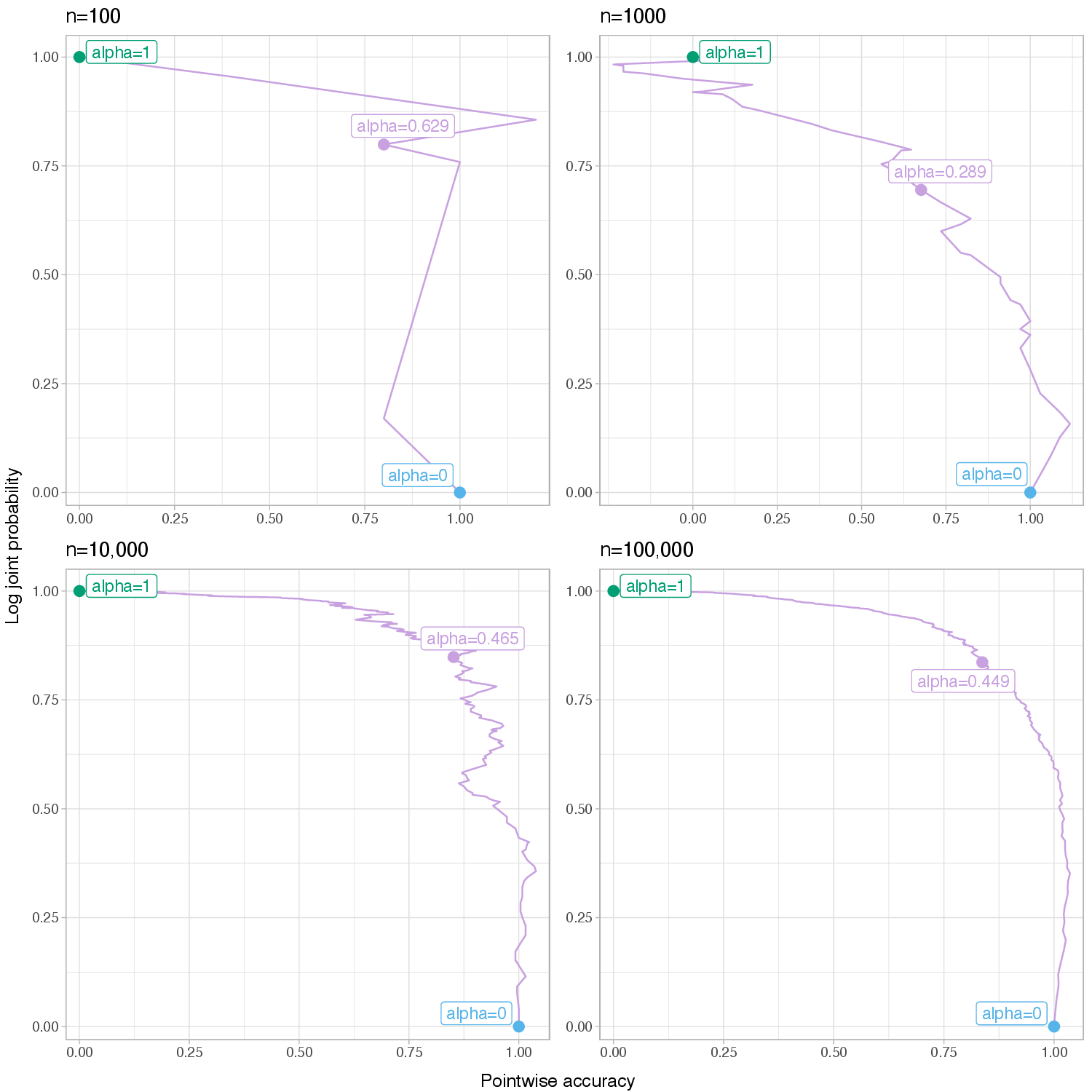}
    \caption{Artemis plots of four simulations with different sequence length~$n$. All simulations are from the medium case   where the Poisson emission rates are given by $\lambda=(15,20,25)$ for the three hidden states.}
    \label{fig:size-of-n}
\end{figure}

In Figure~\ref{fig:artemis_i_iii} the size of the simulated sequences was $n=10^5$. We recommend to choose the length of the simulated sequence length based on visual inspection of the Artemis plots in Figure~\ref{fig:size-of-n}. The Artemis curves for the (scaled) pointwise accuracy and log-joint probability stabilize for increasing values of sequence length. For the medium case with $n=10^5$ the Artemis profile is sufficiently smooth to produce a reliable value of the optimal~$\alpha$. 
\subsection{Robustness of Artemis analyses across HMMs} \label{sec:Robustness}
In this section we investigate the shape of the Artemis plot and variability in the optimal tuning parameter $\alpha$ across different HMMs. We consider a situation where we vary the transition probability matrix and emission probabilities in a 3-state Poisson HMM. The parameters of the model are given by
\begin{equation}
    \pi = (0.8, 0.1, 0.1), \;\;
    \Gamma =
    \begin{pmatrix}
        q & \frac{1 - q}{2} & \frac{1 - q}{2} \\
        \frac{1 - q}{2} & q & \frac{1 - q}{2} \\
        \frac{1 - q}{2} & \frac{1 - q}{2} & q
    \end{pmatrix} 
    \;\; {\rm and} \;\;
    \lambda = (20-a, 20, 20+a),
    \label{TheNineCases}
\end{equation}
and we consider $3\times3=9$ different models with and $q=(0.8,0.5,0.1)$ and $a=(10,5,2)$. 

The Artemis plots for the nine different models are shown in Figure~\ref{fig:change-model}. We observe that the optimal~$\alpha$ depends on the model parameters (the optimal $\alpha$ ranges from 0.123 to 0.45), which shows that there is no single optimal~$\alpha$ that can used for all HMM decoding problems. Thus, it is important to estimate a unique tuning parameter~$\alpha$ for each data set. 
\begin{figure}[!htb]
    \centering
    \includegraphics[width=0.65\linewidth]{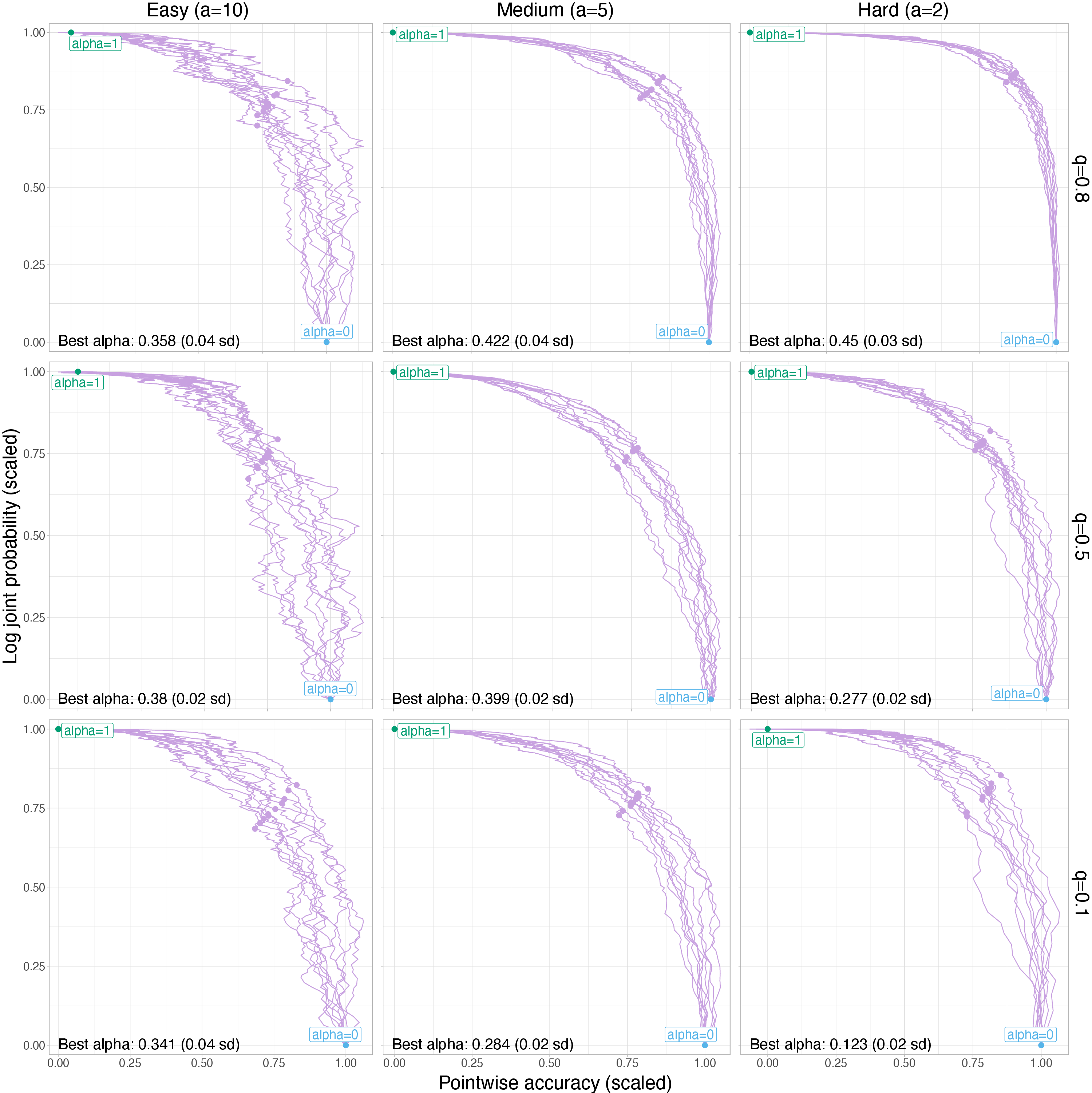}
    \caption{Artemis plots for nine cases of different models. The columns are changes in emission rates, and the rows are changes in transition probabilities, specifically the probability of staying in a state.}
    \label{fig:change-model}
\end{figure}

Figure~\ref{fig:change-model} also shows that when considering each case, the smaller the probability of staying in a state is (the smaller the parameter $q$ is), the closer the optimal $\alpha$ is to Posterior decoding ($\alpha = 0$). This property intuitively makes sense because Posterior decoding is more inclined to produce frequent transitions between states than Viterbi, hence a hybrid path optimizing the pointwise accuracy will be more similar to the path obtained by Posterior decoding than to the path obtained by Viterbi (and vice-versa).
\subsection{Block-wise accuracy for hybrid decoding} \label{sec:Blockwise}
As described in \cite{LemberKolydenko2014} and \cite{KuljusLember2023}, the idea of the hybrid method is to consider larger parts of the hidden state sequence when decoding in order to avoid potential inadmissible paths in posterior decoding. In this section we investigate the block-wise accuracy of hybrid decoding.

Let us return to the situation in (\ref{Cases3statePoisson}) with the easy ($a = 10$), medium ($a = 5$) and hard ($a = 2$) cases of decoding difficulty. In Figure~\ref{fig:block1} we show on the x-axis the block-size. Blocks of size one corresponds to considering each position in the sequence, blocks of size two any two pairs in the sequence, and so on. On the y-axis we show the ability to decode the block. For blocks of size one this is just pointwise accuracy. We see in general that posterior decoding has a higher accuracy for pointwise accuracy (blocks of size one), but for blocks of size two or larger the accuracy of posterior decoding is always lower than hybrid decoding. Actually, we see that hybrid decoding outperforms both Posterior decoding and Viterbi for small blocks larger than or equal to two.
\begin{figure}[!htb]
    \centering
    \includegraphics[width=0.65\linewidth]{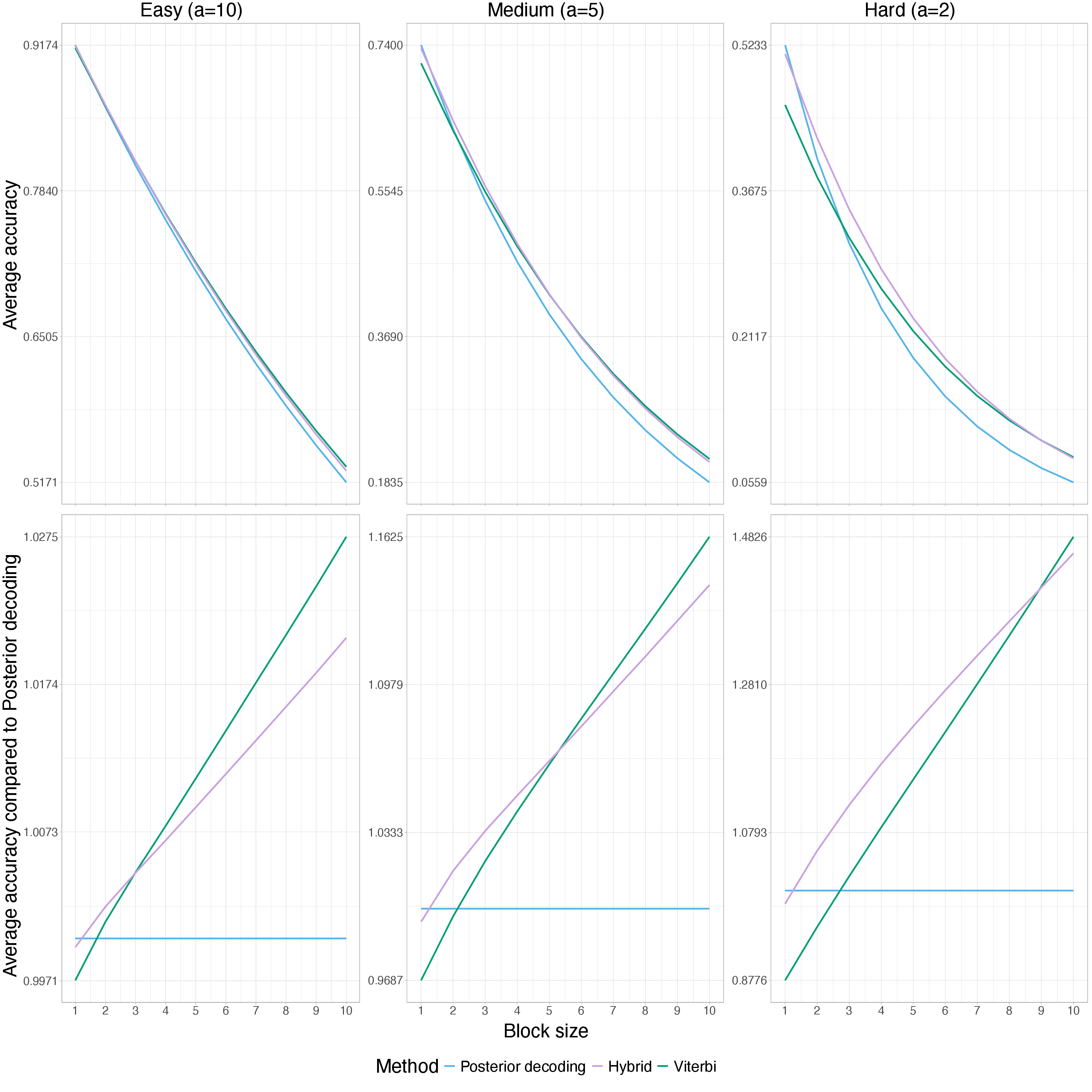}
    \caption{An illustration of the gain of using hybrid on intermediate segments. In the top row, we consider the average accuracy, based on 10 simulations for each block size, for Posterior decoding, hybrid, and Viterbi for three cases of increasing decoding difficulty, as also considered in Figure \ref{fig:artemis_i_iii}. In the bottom row, we consider the average accuracy compared to Posterior decoding instead, to better see the difference between the three methods. The $\alpha$s used for hybrid are the optimal $\alpha$s found using the Artemis plots in Figure \ref{fig:artemis_i_iii}.}
    \label{fig:block1}
\end{figure}

The block-wise accuracies of the three decoding methods on the nine situations from (\ref{TheNineCases}) for increasing block sizes are visualized in Figure~\ref{fig:blocks}. The results are in line with intuition: in terms of accuracy, Posterior decoding is superior when considering pointwise accuracy (blocks of size 1), whereas hybrid outperforms both Posterior decoding and Viterbi for intermediate block sizes, and finally Viterbi is superior for very large block sizes. These results shows exactly how hybrid is an intermediary decoding method between Posterior decoding and Viterbi.
\begin{figure}[!htb]
    \centering
    \includegraphics[width=0.65\linewidth]{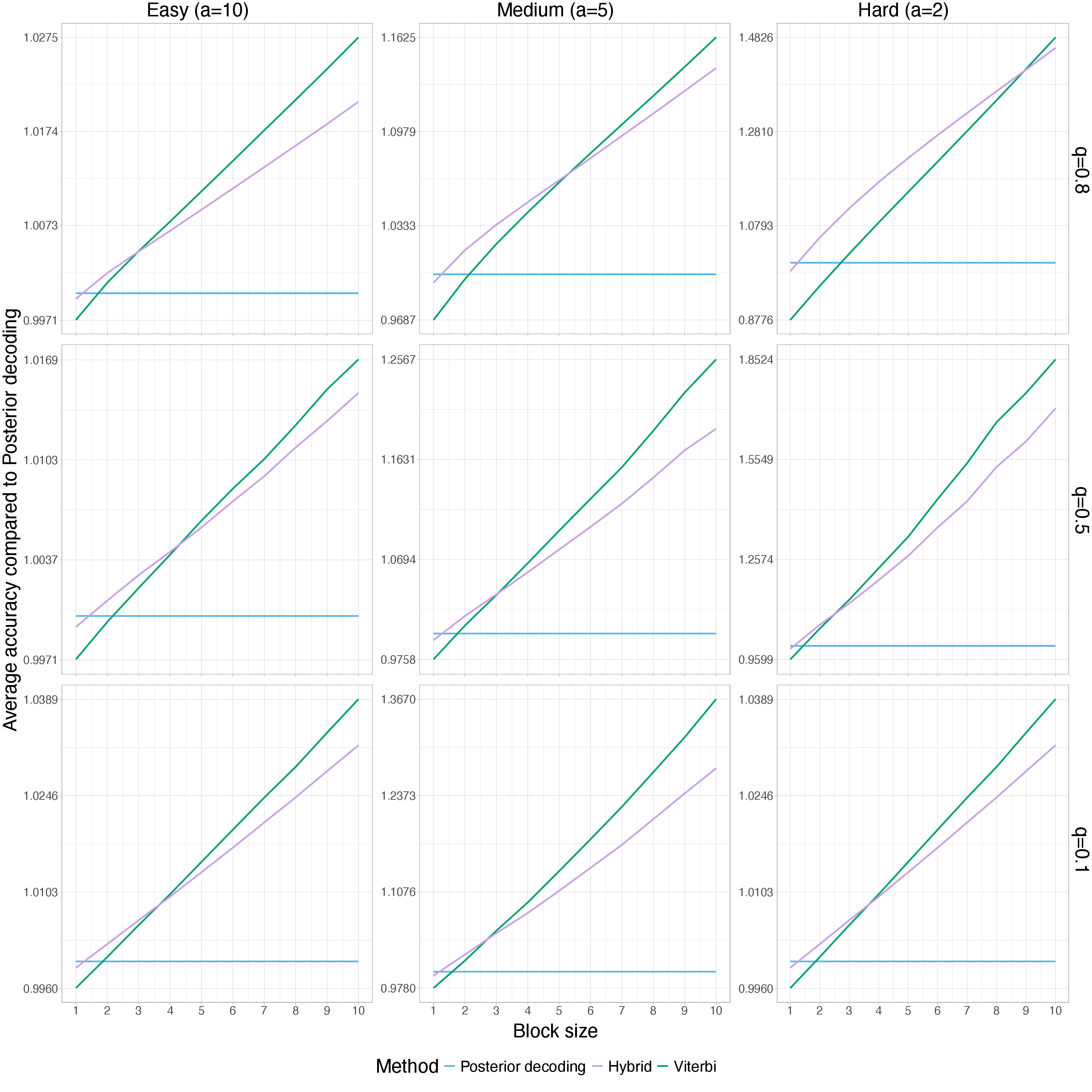}
    \caption{The average accuracy (based on 10 simulations for each block size) compared to Posterior decoding for hybrid and Viterbi for increasing block sizes in six different cases. The columns are changes in emission rates, and the rows are changes in transition probabilities, specifically the probability of staying in a state. The top row is also illustrated in Figure \ref{fig:block1}. The optimal $\alpha$s used for hybrid are found in Figure \ref{fig:change-model}.}
    \label{fig:blocks}
\end{figure}
\section{Discussion} \label{Sec:Discussion}
The fetal lamb movement data discussed in Section~\ref{sec:FMCImotivate} is a rather short sequence of length~225 observations. A number of adaptations are required in order to apply FMCI for long sequences. We describe these adaptations in \cite{CollMacia2025} where we perform a genome-wide FMCI analysis of a 2-state Poisson HMM that identifies regions of archaic introgression into modern humans.

Our descriptions and applications of FMCI is for HMMs with two hidden states and for very simple pattern distributions. We want to emphasize that FMCI can be extended to multiple states and complex patterns, and this extension is indeed available in \cite{AstonMartin2007}. \cite{martin2020distributions} provide an interesting methodology for computing distributions of complicated patterns in categorical time series for very flexible and general models. 

We derived the hybrid risk formulation of \cite{LemberKolydenko2014} using a weighted geometric mean approach. We emphasize that the hybrid risk formulation is in fact a very general framework. Indeed, \cite{LemberKolydenko2014} and \cite{KuljusLember2023} explain how the decision-theoretic approach for segmental classification developed in \cite{yau2013decision} is in fact a special case of the hybrid risk framework. 
\section*{Acknowledgements}
We are indebted to Morten Ravn and Frederik Ligaard Sørensen for initial results on finite Markov chain imbedding.
We are grateful to J\"{u}ri Lember and Kristi Kuljus for an introduction to hybrid segmentation, sharing code for calculating the hybrid path, hosting Zenia Elise Damgaard Bæk during a visit to Tartu (supported by a travel grant from the Danish Data Science Academy), arranging the Markov model workshop in Tartu in the autumn 2024, and valuable discussions. We thank the Novo Nordisk Foundation for funding.
\bibliography{references}
\renewcommand{\thesection}{A} 
\section{Appendix: Hybrid risk and hybrid decoding}
\subsection{Hybrid risk}
In this section we review the risk-based framework introduced in \cite{LemberKolydenko2014} and \cite{KuljusLember2023}. The aim of decoding is to estimate a hidden state sequence~$s=(s_1,\ldots,s_n)$ from the observed sequence~$x=(x_1,\ldots,x_n)$. Let $L(y,u)$ denote the loss function between a hidden state sequence $u=(u_1,\ldots,u_n)$ and the true (but unknown) hidden state sequence $y=(y_1,\ldots,y_n)$. We use the definition of conditional risk from \cite{KuljusLember2023}:
\begin{definition}
    For any estimated state sequence $u$, the expected loss for a given observed sequence $x$ is defined as the conditional risk
    \begin{equation*}
        R(u|x) = \Exp(L(Y, u)|x) = \sum_{y} L(y,u) \Prob(y|x).
    \end{equation*}
\end{definition}
The estimated state sequence $s$ minimizes the conditional risk, i.e.
\begin{equation*}
    s = \argmin_{u} R(u|x).
\end{equation*}
Posterior decoding and Viterbi decoding are directly related to specific loss functions. We first derive the loss function for posterior decoding and second the loss function for Viterbi decoding.

Consider the pointwise $L_1$ loss function 
\begin{equation}
    L_1(y,u) = \frac{1}{n} \sum_{t=1}^n \mathbbm{1}(y_t \neq u_t),
\end{equation}
where $\mathbbm{1}()$ is the indicator function. The conditional risk for $L_1$ is
\begin{align*}
    R_1(u|x) &= \Exp\Big[L_1(Y,u)|x\Big]=\sum_y L_1(y,u)\Prob(y|x)=
    \frac{1}{n} \sum_{t=1}^n \sum_y \Big( 1-\mathbbm{1}(y_t = u_t) \Big) \Prob(y|x) \\
    &= 1-\frac{1}{n} \sum_{t=1}^n \Prob(y_t = u_t|x).
\end{align*}
We conclude that the state sequence that minimizes the conditional risk $R_1$ is the sequence that maximizes the conditional probability $\Prob(y_t = u_t|x)$ for each $t = 1,\dots,n$, which is equivalent to posterior decoding.

Next consider the global loss function $L_\infty$ defined by
\begin{equation}
    L_\infty(y,u) =
    \begin{cases}
        1, & \text{if } y \neq u, \\
        0, & \text{if } y = u.
    \end{cases}
\end{equation}
The conditional risk for $L_\infty$ is
\begin{align*}
    R_\infty(u|x) &= \mathbb{E}\Big[L_\infty(Y,u)|x\Big] 
    =\sum_y L_\infty(y,u) \Prob(y|x) 
    = 1 + \sum_y \Big[ -\big( 1-L_\infty(y,u) \big) \Prob(y|x) \Big] \\
    &= 1-\Prob(y=u| x).
\end{align*}
We conclude that the state sequence that minimizes the conditional risk $R_\infty$ is the sequence that maximizes the conditional probability $\Prob(y=u|x)$, which is equivalent to Viterbi decoding.

In order to combine the conditional risk functions we define the logarithmic risks $\Bar{R}_1$ and $\Bar{R}_\infty$ as
\begin{align*}
    \Bar{R}_1(u|x) = -\frac{1}{n} \sum_{t=1}^n \log \Prob(y_t = u_t|x),\;\;\; {\rm and} \;\;\;
    \Bar{R}_\infty(u|x) = -\frac{1}{n}\log \Prob(y = u|x).
\end{align*}
The logarithmic risks also allow us to interpret the hybrid risk in terms of blocks, as described in \cite{KuljusLember2023}.

We finally combine posterior decoding and Viterbi in the hybrid conditional risk function  
\begin{eqnarray}
    R_{\alpha}(u|x)=(1-\alpha)\Bar{R}_1(u|x) + \alpha \Bar{R}_\infty(u|x),
    \label{eq:hybrid-risk}
\end{eqnarray}
where $\alpha \in [0,1]$ is a tuning parameter. 
\begin{definition}
    \label{def:hybrid-path}
    The hybrid path $s$ is the sequence that minimizes the hybrid risk:
    \begin{equation}
        s=s(\alpha,x)=\argmin_u\{ R_{\alpha}(u|x) \}.
    \end{equation} 
\end{definition}

\cite{LemberKolydenko2014} prove that the hybrid path is always admissible for $\alpha > 0$. That is, the conditional probability $\Prob(s|x)$ of the path is always positive.
\subsection{Hybrid decoding}
In order to determine the hybrid path for a given $\alpha$, a recursive algorithm similar to Viterbi is used. The following theorem is stated in \cite{KuljusLember2023}, and is formulated again here for completeness.
\begin{theorem}
    The hybrid path can be found by the following recursion and back-tracking. Define the hybrid table $\delta$ by the initial values 
    \begin{align*}
        \delta_1(j) &= \alpha\log \Prob(y_1 = j,x_1) + (1-\alpha)\log\Prob(y_1 = j|x), 
    \end{align*}
    and by the recursion    
    \begin{align*}
        \delta_t(j) &= 
        \max_i \big\{ \delta_{t-1}(i) + \alpha\log\Prob(y_t = j,x_t|y_{t-1} = i) \big\} + (1-\alpha)\log\Prob(y_t = j \mid x),
    \end{align*}
    for all $j$ and $t = 2,\dots,n$. In the second term, the probability $\Prob(y_t = j|x)$ is the conditional state probability calculated from the forward-backward tables and used in Posterior decoding.

    When filling the table $\delta_t(j)$, define the pointers $\psi_t(j)$ as
    \begin{align*}
        \psi_t(j) &= \argmax_{i} \big\{ \delta_t(i) + \alpha\log\Prob(y_{t+1} = j,x_{t+1}|y_t = i) \big\}, \quad t = 1,\dots,n-1, \\
        \psi_n &= \argmax_i \ \delta_n(i).
    \end{align*}
    The path $s$ minimising the hybrid risk is obtained by
    \begin{equation*}
        s_n = \psi_n, \quad s_t = \psi_t(s_{t+1}), \quad t=n-1,\dots,1.
    \end{equation*}
\end{theorem}

\begin{proof}
    By Definition \ref{def:hybrid-path}, the aim is to find the state path $s$ minimising the hybrid risk (\ref{eq:hybrid-risk}). We start by noticing that the minimization of the hybrid risk is equivalent to the maximization
    \begin{equation*}
        \max_s \Big\{ \alpha\log\Prob(y = s \mid x) + (1-\alpha) \sum_{t=1}^n \log\Prob(y_t = s_t \mid x) \Big\}.
    \end{equation*}
    The idea is to rewrite the above maximisation such that it is similar to the entries in the hybrid table~$\delta_t(j)$ for $t = 1,\dots,n$. We start by considering
    \begin{align*}
        \MoveEqLeft[6] \max_s \Big\{ \alpha\log\Prob(y = s \mid x) + (1-\alpha)\sum_{t=1}^n \log\Prob(y_t = s_t \mid x) \Big\} \\
        &= \max_s \big\{ \alpha\log\Prob(y = s,x) + (1-\alpha)\sum_{t=1}^n \log\Prob(y_t = s_t \mid x) \big\} \\
        &= \max_s \big\{ \alpha\log\Prob(y_1 = s_1,x_1) + \alpha\sum_{t=2}^n \log\Prob(y_t = s_t,x_t \mid y_{t-1} = s_{t-1}) \\
        &\qquad\qquad + (1-\alpha)\sum_{t=1}^n \log\Prob(y_t = s_t \mid x) \big\} \\
        &= \max_s \big\{ \alpha\log\Prob(y_1 = s_1,x_1) + (1-\alpha)\log\Prob(y_1 = s_1 \mid x) \\
        &~~~~~~~~~~~ + \sum_{t=2}^n \big[ \alpha\log\Prob(y_t = s_t,x_t \mid y_{t-1} = s_{t-1}) + (1-\alpha)\log\Prob(y_t = s_t \mid x) \big] \big\} \\
        &= \max_s \big\{ \delta_1(s_1) + \sum_{t=2}^n \big[ \alpha\log\Prob(y_t = s_t,x_t \mid y_{t-1} = s_{t-1}) + (1-\alpha)\log\Prob(y_t = s_t \mid x) \big] \big\}.
    \end{align*}
    If we consider an arbitrary $t$, note that the above is equal to
    \begin{align*}
        \MoveEqLeft[6] \max_{s_{1:t}} \big\{ \delta_1(s_1) + \sum_{i=2}^t \big[ \alpha\log\Prob(y_i = s_i,x_i \mid y_{i-1} = s_{i-1}) + (1-\alpha)\log\Prob(y_i = s_i \mid x) \big] \big\} \\
        &= \max_j \big\{ \max_{s_{1:(t-1)}, s_t = j} \big\{ \delta_1(s_1) + \sum_{i=2}^{t-1} \big[ \alpha\log\Prob(y_i = s_i,x_i \mid y_{i-1} = s_{i-1}) \\
        &\qquad\qquad\qquad\qquad\qquad\qquad\quad\qquad + (1-\alpha)\log\Prob(y_i = s_i \mid x) \big] \\
        &~~~~~~~~~~~~~~~~~~~~~~~~~~ + \alpha\log\Prob(y_t = j,x_t \mid y_{t-1} = s_{t-1}) + (1-\alpha)\log\Prob(y_t = j \mid x) \big\} \big\} \\
        &= \max_j \ \delta_t(j).
    \end{align*}
    Thus, we conclude that we solve the minimisation of the hybrid risk (\ref{eq:hybrid-risk}) by solving the maximisation
    \begin{equation*}
        \max_j \ \delta_n(j).
    \end{equation*}
    This is similar to the recursion used when calculating the Viterbi path, and thus the state path minimising the hybrid risk is obtained by backtracking from $\psi_n$.
\end{proof}

Since the theorem is applicable for pairwise Markov models, it can be simplified in the case of HMMs. That is, the scores $\delta_t$ can be rewritten as
\begin{align*}
    \delta_1(j) &= \alpha \log(\Phi(x_1|j) \pi_j) + (1-\alpha) \log \Big( \frac{\beta_t(j)\alpha_t(j)}{\Prob(x)} \Big), \\
    \delta_t(j) &= \max_i \big\{ \delta_{t-1}(i) + \alpha \log (\Phi(x_t|j)\Gamma_{ij}) \big\} + (1-\alpha) \log \Big( \frac{\beta_t(j)\alpha_t(j)}{\Prob(x)} \Big),
\end{align*}
where $\alpha_t(j) = \Prob (y_t=j, x_1, \dots,x_t)$ is the table of forward probabilities and $\beta_t(j)=\Prob(x_{t+1},\ldots,x_n|y_t=j)$ is the table of backward probabilities. 
\end{document}